\documentclass{article} 
\usepackage[preprint]{colm2026_conference}

\usepackage{microtype}
\usepackage{hyperref}
\usepackage{url}
\usepackage{booktabs}
\usepackage{placeins}

\usepackage{array}
\usepackage[table]{xcolor}
\usepackage{tabularx}

\usepackage{lineno}

\definecolor{darkblue}{rgb}{0, 0, 0.5}
\hypersetup{colorlinks=true, citecolor=darkblue, linkcolor=darkblue, urlcolor=darkblue}

\usepackage{tcolorbox}

\usepackage{graphicx}
\usepackage{subcaption}
\usepackage{bm}
\usepackage{xparse}

\usepackage{algorithm}
\usepackage{algorithmic}

\usepackage{amsmath}
\usepackage{amssymb}
\usepackage{mathtools}
\usepackage{multirow}
\usepackage{xcolor}
\usepackage{amsthm}
\usepackage{thm-restate}
\usepackage{tikz}
\usepackage{calc}
\newsavebox{\methodbox}
\newsavebox{\resultbox}
\usetikzlibrary{arrows.meta,positioning,calc,shapes}

\usepackage[capitalize,noabbrev]{cleveref}
\usepackage{pgffor} 

\renewcommand{\arraystretch}{1.5} 

\NewDocumentCommand{\phimg}{ s m m }{%
  \IfFileExists{#3}{%
    \IfBooleanTF{#1}{%
      \includegraphics[width=#2]{#3}%
    }{%
      \fcolorbox{blue}{white}{%
        \includegraphics[width=#2]{#3}%
      }%
    }%
  }{%
    \IfBooleanTF{#1}{%
      \fcolorbox{gray!40}{gray!20}{%
        \rule{0pt}{#2}\rule{#2}{0pt}%
      }%
    }{%
      \fcolorbox{blue}{gray!20}{%
        \rule{0pt}{#2}\rule{#2}{0pt}%
      }%
    }%
  }%
}

\NewDocumentCommand{\labeledimg}{ s m m m }{%
  \begin{tabular}{@{}c@{}}\\[-1em]
    \textbf{#2} \\[-0.2em]
    \IfBooleanTF{#1}{%
      \phimg*{#3}{#4}%
    }{%
      \phimg{#3}{#4}%
    }%
  \end{tabular}%
}

\definecolor{citecolor}{HTML}{2779af}
\definecolor{linkcolor}{HTML}{c0392b}
\hypersetup{breaklinks=true,colorlinks,citecolor=citecolor,linkcolor=linkcolor}

\definecolor{customLightPurple}{rgb}{1, 0.85, 1}
\definecolor{customMediumPurple}{rgb}{1, 0.3, 1}
\definecolor{customDarkPurple}{rgb}{1, 0.1, 1}

\definecolor{customLightBlue}{rgb}{0.9, 0.9, 1}
\definecolor{customMediumBlue}{rgb}{0.6, 0.6, 1}
\definecolor{customDarkBlue}{rgb}{0.33, 0.33, 1}
\definecolor{customDarkestBlue}{rgb}{0.01, 0.01, 1}

\definecolor{customLightRed1}{rgb}{1, 0.9, 0.9}
\definecolor{customMediumRed1}{rgb}{1, 0.6, 0.6}
\definecolor{customDarkRed1}{rgb}{1, 0.4, 0.4}
\definecolor{customDarkestRed1}{rgb}{1, 0.05, 0.05}

\definecolor{customLightRed2}{rgb}{1, 0.85, 0.85}
\definecolor{customMediumRed2}{rgb}{1, 0.65, 0.65}
\definecolor{customDarkRed2}{rgb}{1, 0.43, 0.43}
\definecolor{customDarkestRed2}{rgb}{1, 0.03, 0.03}

\definecolor{customLightRed3}{rgb}{1, 0.80, 0.80}
\definecolor{customMediumRed3}{rgb}{1, 0.5, 0.5}
\definecolor{customDarkRed3}{rgb}{1, 0.2, 0.2}
\definecolor{customDarkestRed3}{rgb}{1, 0.01, 0.01}

\definecolor{customLightC1}{rgb}{0, 0.9, 0.9}
\definecolor{customMediumC1}{rgb}{0, 0.6, 0.6}
\definecolor{customDarkC1}{rgb}{0, 0.4, 0.4}
\definecolor{customDarkestC1}{rgb}{0, 0.05, 0.05}

\definecolor{customLightC2}{rgb}{0, 0.85, 0.85}
\definecolor{customMediumC2}{rgb}{0, 0.65, 0.65}
\definecolor{customDarkC2}{rgb}{0, 0.43, 0.43}
\definecolor{customDarkestC2}{rgb}{0, 0.03, 0.03}

\definecolor{customLightC3}{rgb}{0, 0.80, 0.80}
\definecolor{customMediumC3}{rgb}{0, 0.5, 0.5}
\definecolor{customDarkC3}{rgb}{0, 0.2, 0.2}
\definecolor{customDarkestC3}{rgb}{0, 0.01, 0.01}

\theoremstyle{plain}
\newtheorem{theorem}{Theorem}[section]

\newtheorem{lemma}[theorem]{Lemma}

\theoremstyle{definition}

\newtheorem{assumption}[theorem]{Assumption}
\theoremstyle{remark}

\newcommand{\ourmethodshort}{DPG}
\newcommand{\ourmethod}{the Dataset Policy Gradient}
\newcommand{\ourmethodplural}{Dataset Policy Gradients}
\newcommand{\ourmethodsingular}{Dataset Policy Gradient}

\usepackage[normalem]{ulem} 

\newif\ifshowcomments
\showcommentsfalse 

\ifshowcomments
    \usepackage[colorinlistoftodos, textsize=footnotesize]{todonotes} 
    
    \usepackage[paperwidth=11in, right=2.5in, marginparwidth=2in]{geometry}
    
    \newcommand{\del}[1]{\textcolor{red}{\sout{#1}}}

\else
    \usepackage[disable]{todonotes} 
    
    \newcommand{\del}[1]{}        
      
\fi

\newcommand{\sidecomment}[2][]{%
    \todo[color=customLightBlue, bordercolor=customDarkBlue, linecolor=customDarkBlue, #1]{#2}%
}

\title{Synthetic Data for any Differentiable Target}

\author{Tristan Thrush, Sung Min Park, Herman Brunborg, Luke Bailey, \\ \bfseries Marcel Roed, Neil Band, Christopher Potts \& Tatsunori Hashimoto\\
Stanford University\\
\texttt{\{tthrush,cgpotts,thashim\}@stanford.edu} \\
}

\begin{document}

\ifcolmsubmission
\linenumbers
\fi

\maketitle

\begin{abstract}
What are the limits of controlling language models via synthetic training data? We develop a reinforcement learning (RL) primitive, \ourmethod\ (\ourmethodshort), which can precisely optimize synthetic data generators to produce a dataset of targeted examples. When used for supervised fine-tuning (SFT) of a target model, these examples cause the target model to do well on a differentiable metric of our choice. Our approach achieves this by taking exact data attribution via higher-order gradients and using those scores as policy gradient rewards. We prove that this procedure closely approximates the true, intractable gradient for the synthetic data generator.
To illustrate the potential of \ourmethodshort, we show that, using only SFT on generated examples, we can cause the target model's LM head weights to
(1) embed a QR code, (2) embed the pattern \texttt{67}, and (3) have lower $\ell^2$ norm. We additionally show that we can cause the generator to (4) rephrase inputs in a new language and (5) produce a specific UUID, even though neither of these objectives is conveyed in the generator's input prompts.
These findings suggest that \ourmethodshort\ is a powerful and flexible technique for shaping model properties using only synthetic training examples.
\end{abstract}

\section{Introduction}

Synthetic training data has recently gained significant interest \citep{wang2023selfinstructaligninglanguagemodels, alpaca, yang2025synthetic, ruan2025reasoninglearnlatentthoughts} but how finely can we control synthetic data generation? It is well-attested that training examples (real and synthetic) can communicate unexpected information to language models even in the context of simple supervised fine-tuning (SFT). Recent prominent examples include %
emergent misalignment \citep{betley2026emergent,chua2025thought}, %
subliminal learning \citep{cloud2025subliminal,betley2025weird}, %
data poisoning from harmless inputs \citep{kong2025revisiting}, %
and %
model provenance \citep{kuditipudi2025blackbox}. %
Is there a way to tractably train a synthetic data generator that produces training data targeting any phenomena we choose?

Intuitively, straightforward reinforcement learning techniques could be used to optimize synthetic data generators directly for downstream metrics. Every time a dataset is generated by our policy, we could train a model on it and measure a metric of interest from the model. We could then use this metric as a single reward for the entire dataset and perform a policy gradient step. However, this approach is computationally prohibitive because it provides only a single reward for a full run of inner target model training and evaluation.

In this work, we present \ourmethod\ (\ourmethodshort), a principled RL approach that enables us to generate synthetic training data for any differentiable downstream target. With our method, rewards are at the level of individual synthetic texts, instead of the dataset level. This method opens the door to a wide range of applications in which training examples are chosen or synthesized with the goal of imbuing a target model with a specific property.

Our approach leverages the meta-learning results of \citet{raghu2021metalearning}, and the recent improvements from \citet{engstrom2025metagradients}. These papers demonstrate how to compute metagradients (gradients of hyperparameters of the training process) tractably at the scale of LLM training. 
The metagradient enables backpropagation from a differentiable post-training metric (e.g., loss on a benchmark) to parameters of the training process (e.g., optimization hyperparameters such as learning rate schedules).
Importantly, it is also tractable to compute metagradients for training example weights, if training occurs with a data-weighted loss. 
This leads to the key insight for our method:
we can incorporate this metagradient-based data valuation approach into an RL procedure to generate targeted synthetic training data.

The \ourmethodshort\ approach is a flexible framework.
For the experiments in this paper, we 
use the configuration in 
\cref{fig:fig1}, top: a generator creates a pool of synthetic examples $D$, which are the inputs to learning algorithm $\mathcal{A}$. This learning algorithm trains a target LM on $D$ with example-level training loss weights $w_{i}$ set to $1$. Then, the target LM is evaluated against a differentiable metric, $\Phi$. The metagradient of $\Phi$ with respect to the $w_{i}$s determines a reward that is used to update the generator using Group Relative Policy Optimization (GRPO) \citep{shao2024grpo}. The trained generator produces examples that, if used to train a target LM with standard SFT, lead that LM to do well on $\Phi$. In \cref{sec:theory}, we prove that the resulting policy gradient of this approach approximates the desired intractable policy gradient \sidecomment{NB: I wonder if ``intractable policy gradient'' could be made more precise so people reading quickly will get it. Maybe something like ``approximates the desired but intractable policy gradient: the gradient of the generator parameters with respect to the expected downstream metric.''} for the synthetic data generator, under reasonable smoothness assumptions.

\begin{figure*}[!t]
  \centering
  \resizebox{0.8\linewidth}{!}{
  \savebox{\methodbox}{%
    \begin{minipage}[t]{1.0\textwidth}
      \centering
      \sf

      \Large \underline{\textbf{\ourmethodsingular}}

      \resizebox{\linewidth}{!}{%
        \begin{tikzpicture}[
          font=\sffamily,
          >=Latex,
          box/.style={draw, very thick, rounded corners=4pt, minimum height=2.69cm, minimum width=3.2cm, align=center},
          flow/.style={->, very thick},
          dashback/.style={->, very thick, dashed},
        ]

          \node[box] (teacher) {\Large Generator};

          \node[inner sep=0pt] (data) [right=2.69cm of teacher] {%
            \begin{tikzpicture}[x=1cm,y=1cm,line width=1.6pt]
              \def\w{1.6}
              \def\h{1.6}
              \def\r{0.35}

              \draw (0,0) ellipse (\w/2 and \r);

              \draw (-\w/2,0) -- (-\w/2,-\h + 0.175);
              \draw (\w/2,0) -- (\w/2,-\h + 0.175);

              \foreach \y in {-0.45,-0.9,-1.35} {
                \draw (\w/2,\y)
                  arc[start angle=0, end angle=-180,
                      x radius=\w/2, y radius=\r];
              }
            \end{tikzpicture}
          };

          \node[below=0cm of data, align=center] {\Large Synthetic Data, $D$};

          \node[box] (student) [right=2.69cm of data] {%
            \Large Learning Alg. $\mathcal{A}$, trained on\\
            \Large $x_i \in D$ with loss $w_i|_{w_i=1}\ell(x_i)$\\
            ($w_i$'s do not effect training)%
          };

          \node[inner sep=0pt] (evalicon) [right=2.69cm of student] {%
            \begin{tikzpicture}[x=1cm,y=1cm,line width=1.6pt]
              \draw[rounded corners=2pt] (0,0) rectangle (1.05,1.25);

              \draw (0.18,0.95) -- (0.75,0.95);
              \draw (0.18,0.78) -- (0.88,0.78);
              \draw (0.18,0.61) -- (0.70,0.61);

              \draw[line width=3.2] (0.80,0.55) circle (0.25);
              \draw[line width=3.2] (0.93,0.38) -- (1.22,0.10);
            \end{tikzpicture}%
          };

          \node[below=0.25cm of evalicon, align=center] {\Large Differentiable metric, $\Phi$};

          \draw[flow] (teacher.east) -- (data.west);
          \draw[flow] (data.east) -- (student.west);
          \draw[flow] (student.east) -- (evalicon.west);

          \draw[dashback]
            (evalicon.north)
            .. controls ($(evalicon.north west)+(0,0.9)$) and ($(student.north east)+(0,2.2)$) ..
            (student.north)
            node[midway, above=2pt, align=center] {\Large Metagradient\\\Large Backprop};

          \draw[dashback]
            (student.north)
            .. controls ($(student.north west)+(2.2,0.9)$) and ($(data.north east)+(2.2,2.2)$) ..
            (data.north)
            node[midway, above=2pt, align=center] {%
              \Large Set rewards for $D$ to be:\\
              \Large $\nabla_{\bm{w}} \Phi(\mathcal{A}(\bm{w}, D))|_{\bm{w}=\bm{1}}$%
            };

          \draw[dashback]
            (data.north)
            .. controls ($(data.north west)+(0,1.3)$) and ($(teacher.north east)+(0,1.7)$) ..
            (teacher.north)
            node[midway, above=2pt, align=center] {\Large Optimize with\\\Large RL Objective};

        \end{tikzpicture}%
      } 
    \end{minipage}%
  } 

\savebox{\resultbox}{%
  \begin{minipage}[t]{1.0\textwidth}
    \centering
    \sf

    \Large \underline{\textbf{Example \ourmethodsingular\ Result}}

    \resizebox{\linewidth}{!}{%
      \begin{tikzpicture}[
        font=\sffamily,
        >=Latex,
        box/.style={draw, very thick, rounded corners=4pt, minimum height=2.69cm, minimum width=3.2cm, align=center},
        flow/.style={->, very thick},
      ]

\node[box] (teacher) {\Large Trained\\\Large Generator};

\node[inner sep=0pt] (data) [right=2.69cm of teacher] {%
  \begin{tikzpicture}[x=1cm,y=1cm,line width=1.6pt]
    \def\w{1.6}
    \def\h{1.6}
    \def\r{0.35}
    \draw (0,0) ellipse (\w/2 and \r);
    \draw (-\w/2,0) -- (-\w/2,-\h + 0.175);
    \draw (\w/2,0) -- (\w/2,-\h + 0.175);
    \foreach \y in {-0.45,-0.9,-1.35} {
      \draw (\w/2,\y)
        arc[start angle=0, end angle=-180,
            x radius=\w/2, y radius=\r];
    }
  \end{tikzpicture}
};
\node[below=0cm of data, align=center] {\Large Synthetic Data\\ \Large (Wikipedia Rephrases)};

\node[box, minimum width=5cm, minimum height=3.4cm] (gpttwo) [right=2.69cm of data] {};

\draw[flow] (teacher.east) -- (data.west);
\draw[flow] (data.east) -- (gpttwo.west);

\node[
  draw=blue,
  rounded corners=4pt,
  line width=0.8pt,
  inner sep=4pt,
  font=\Large,
  align=center,
  fill=blue!5,
  anchor=north
] (wikitext) at ($(data.south)+(0,-1.35)$)
{$\cdots$The life and career of Jose \\ Cabalum Sr. (1915/1916-2006)$\cdots$};

\pgfmathsetmacro{\bw}{4.5}
\pgfmathsetmacro{\bh}{0.58}
\pgfmathsetmacro{\gap}{0.20}
\pgfmathsetmacro{\ysoft}{1.45}
\pgfmathsetmacro{\ylm}{\ysoft - \bh - \gap}
\pgfmathsetmacro{\yblk}{\ylm   - \bh - \gap}
\pgfmathsetmacro{\yemb}{\yblk  - \bh - \gap}
\pgfmathsetmacro{\glassx}{1.55}
\pgfmathsetmacro{\glassy}{\ylm - \bh/2}

\begin{scope}[shift={(gpttwo.center)}]

  \node[font=\Large] at (0,2.1) {GPT-2};

  \newcommand{\gptblock}[2]{%
    \draw[very thick, rounded corners=4pt] (-\bw/2,#1) rectangle (\bw/2,#1-\bh);
    \node[font=\Large, align=center] at (0,#1-\bh/2) {#2};
  }

  \gptblock{\ysoft}{Softmax}
  \gptblock{\ylm}{LM Head}
  \gptblock{\yblk}{Transformer Blocks}
  \gptblock{\yemb}{Embed}

  \coordinate (glasscenterLocal) at (\glassx,\glassy);
  \def\lensR{0.32}
  \draw[red, line width=3.2pt] (glasscenterLocal) circle (\lensR);
  \draw[red, line width=2.2pt, line cap=round]
    ($(glasscenterLocal)+(0.23,-0.23)$) -- ($(glasscenterLocal)+(0.56,-0.56)$);
  \draw[red, line width=3.2pt, line cap=round]
    ($(glasscenterLocal)+(0.38,-0.38)$) -- ($(glasscenterLocal)+(0.72,-0.72)$);

\end{scope}

\path (gpttwo.center) ++(\glassx,\glassy) coordinate (glasscenterAbs);

\node[
  draw=red,
  line width=3.5mm,
  inner sep=0pt,
  anchor=west
] (patch) [right=2.69cm of gpttwo]
{\includegraphics[width=3.2cm]{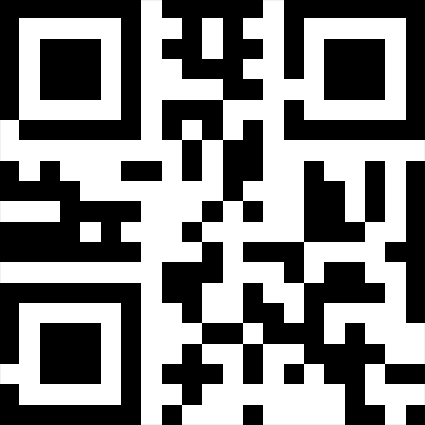}};

\node[font=\Large, align=center, anchor=north] at ($(patch.south)+(0,-2mm)$)
{\textbf{\textcolor{red}{QR code encoded in LM Head}}\\\textbf{\textcolor{red}{after standard continued pre-}}\\\textbf{\textcolor{red}{training on the synthetic data}}};

\draw[red, line width=3pt, densely dotted] (glasscenterAbs) -- (patch.west);

          {\LARGE \textbf{\textcolor{red}{Residual QR Code in LM Head}}};

      \end{tikzpicture}%
    } 
  \end{minipage}%
} 

\noindent
\centering
\begin{tikzpicture}[inner sep=0pt, outer sep=0pt]

  \pgfmathsetlengthmacro{\gap}{8pt}
  \pgfmathsetlengthmacro{\rulew}{0.8pt}

  \node[anchor=north] (L) {\usebox{\methodbox}};
  \node[anchor=north] (R) at ($(L.south)+(0,-0.5)$) {\usebox{\resultbox}};



\end{tikzpicture}
}

  \caption{\ourmethodplural\ allow us to generate synthetic training data for any differentiable target. For example, our generator can learn to generate special Wikipedia article rephrases. When used for continued pretraining of GPT-2, these rephrases turn the upper left 21x21 patch of GPT-2's LM head weight matrix into the QR code seen here (when subtracted from the initial weights, sign'd, and visualized as a greyscale image). The text sample in this figure is the first item in the synthetic dataset, which we generated with a temperature of 1 (i.e., noisy data still produces the result).}
  \label{fig:fig1}
\end{figure*}
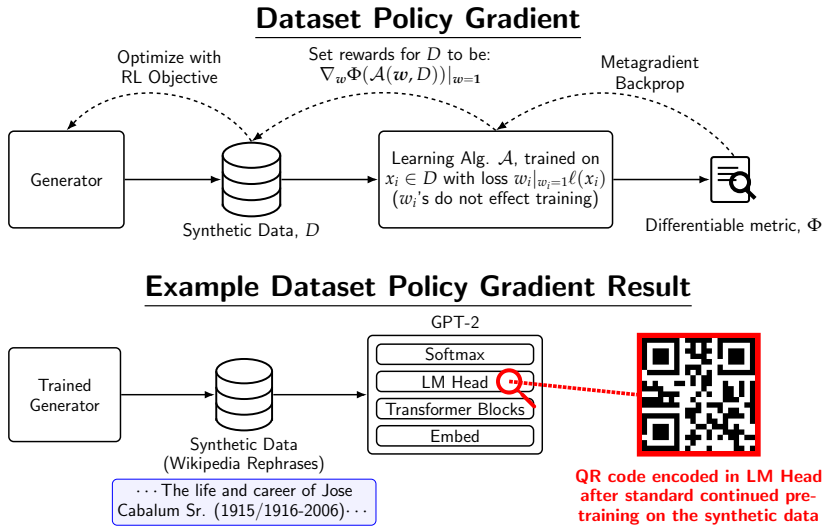

We seek to test the limits of our method by experimenting with unusual choices of $\Phi$. In our first experiments, we demonstrate that the generator produces examples that have a specific effect on the target model: encoding a QR code (\cref{sec:qr}) and the pattern \texttt{67} (\cref{sec:67}) in the LM head weights of the target model, and lowering the $\ell^{2}$ norm of the LM head weights (\cref{sec:norms}). We then directly assess the generator, showing that \ourmethod\ can guide it to rephrase Wikipedia articles in a new language (\cref{sec:multilingual}) and produce a specific UUID (\cref{sec:uuid}), without any prompting for these behaviors.

In our experiments, we perform ablations to disentangle which aspects of the metagradient computation are essential in driving performance. For our QR code, \texttt{67}, and $\ell^2$ norm experiments, we find that computing metagradients with respect to several gradient descent steps of target model training is helpful. For the other experiments, we used a larger model as our target model and only tried one step of target model training for metagradient computation, due to compute constraints.
We also find that the choice of target model optimizer (Adam vs.~SGD) in the computation of the metagradient is a significant factor. Where we used SGD in learning algorithm~$\mathcal{A}$ (\cref{fig:fig1}), the trained generator's synthetic data did not cause the target model to perform well on $\Phi$ (even if Adam was used in after-the-fact training of the target model), whereas Adam is successful in this role. In the single step case for SGD, the metagradient reduces to standard gradient-of-target and gradient-of-train dot-product approximations to influence functions \citep{koh2017influence}. By contrast, where Adam is the optimizer, there are additional terms which make the metagradient different from approximations to typical influence functions, even in the single-step case. This indicates that full metagradients are critical to optimizing the generator.

Overall, our results provide evidence that the \ourmethodshort\ framework allows for a new level of fine-grained control in synthetic training data generation, for the purpose of imbuing downstream models with specific properties -- both desirable and undesirable. 

\section{Related Work}

\paragraph{Synthetic data for language model training.}
Synthetic data is increasingly viewed as a key resource for language model performance gains \citep{wang2023selfinstructaligninglanguagemodels, alpaca, maini2024rephrasingwebrecipecompute, abdin2024phi4technicalreport, ruan2025reasoninglearnlatentthoughts, yang2025syntheticbootstrappedpretraining}. Our contribution is orthogonal: instead of asking what synthetic data heuristics improve performance, we study how precisely synthetic data can be optimized -- via metagradients -- to induce targeted and even unconventional differentiable properties in trained models.

\paragraph{Training data attribution.}
We benefit from work attributing model behavior to individual training examples. Influence functions \citep{Hampel_1974,koh2017influence,bae2022influence} provide local estimates of how upweighting a training datum affects downstream performance. Recent work scales attribution ideas to modern LMs and multi-step training \citep{raghu2021metalearning, pmlr-v162-ilyas22a, pmlr-v202-park23c, grosse2023studyinglargelanguagemodel, xia2024lessselectinginfluentialdata, thrush2025improvingpretrainingdatausing, thudi2025mixminfindingdatamixtures, engstrom2025metagradients, calian2025dataratermetalearneddatasetcuration}. Data attribution is a subroutine in our work: we leverage the metagradients approach from \cite{engstrom2025metagradients} to assign rewards to synthetic training examples generated by an RL policy.

\paragraph{Optimizing and editing training data.}
We focus on generating discrete synthetic training data from scratch. Other work has focused on targeted optimization of perturbations in differentiable training data, such as perturbing existing images \citep{such2019generativeteachingnetworksaccelerating, wang2020datasetdistillation, huang2021metapoisonpracticalgeneralpurposecleanlabel, rosser2026infusionshapingmodelbehavior}. In the discrete data space, recent work includes RL approaches where models iteratively improve by generating synthetic training data for themselves, or through generating some other self-edit. In SEAL \citep{zweiger2025selfadaptinglanguagemodels} LLMs generate candidate self-edits (directives on how to update their own weights); these directives are carried out and edited LLMs are evaluated on downstream tasks. The performances of the edited LLMs are used directly as RL rewards, but this is intractable for our data generation tasks. MASS \citep{kaya2026testtimemetaadaptationselfsynthesis} performs bilevel meta-adaptation using self-synthesized data at test time, computing a training data metagradient within an RL loop. MASS focuses on single datum adaptation at test time and computes the metagradient in the local one-train-step case without taking into account optimizer dynamics, analogous to an influence function approximation which lacks the more general metagradient critical for our tasks. In contrast to these methods, we prove that per-step metagradients provide accurate gradient signals that approximate the intractable full RL problem. Then, we optimize a policy that produces an entirely new training dataset targeting arbitrary differentiable training or post-training properties of an arbitrary target model over multiple training steps, taking into account arbitrary optimizers such as Adam \citep{kingma2015adam}.

\paragraph{Optimizing inference data.}
Several approaches optimize prompts to elicit targeted behaviors at inference time \citep{zou2023universaltransferableadversarialattacks, zhou2023largelanguagemodelshumanlevel, agrawal2026gepareflectivepromptevolution}. We instead optimize the generation of training data, so that learning itself induces desired behaviors.

\section{Methods}

We train a policy (i.e. the generator, $\pi_\theta$) to generate training data for another model (i.e. the target model, trained in the RL loop within $\mathcal{A}$). The objective is to generate synthetic data $D$ that increases the metric, $\Phi(\mathcal{A}(D))$. Formally, we want to optimize $\pi_\theta$ via the objective
\[
\max_{\pi_\theta} \mathbb{E}_{D \sim \pi_\theta}[\Phi(\mathcal{A}(D))],
\]
but a direct approach is expensive: it involves a single RL reward over a dataset instead of a reward for each example in the dataset. In principle, the computational cost could be thousands of times greater than a typical LM RL problem. Could we reduce this to a typical, per-example, RL problem?
Ideally, we want per-example rewards $r(x)$, for $x$ in $D$, such that:
\[
\nabla_\theta \mathbb{E}_{D \sim \pi_\theta}[\Phi(\mathcal{A}(D))] = \mathbb{E}_{D \sim \pi_\theta}\left[\sum_{x \in D} r(x) \nabla_\theta \log \pi_\theta(x)\right].
\]
That is, taking a policy gradient step with respect to our per-example reward is equivalent to taking the intractable policy gradient step. This turns out to be possible and tractable. If $r(x)$ is defined as the exact influence of example $x$ on the reward $\Phi(\mathcal{A}(D))$ through the training process, then the per-example policy gradient closely approximates the dataset-level policy gradient. In the next sections, we elaborate on how to take this exact influence (Section \ref{sec:algorithm}) and prove that this approximation is valid under natural assumptions (Section \ref{sec:theory}).

\subsection{Algorithm}\label{sec:algorithm}

For our experiments, we use Group Relative Policy Optimization (GRPO) to train the generator \citep{shao2024grpo}, as shown in Algorithm \ref{alg:iterative-grpo}. For every outer GRPO step, we can divide the set of policy generations into $G$ training sets for a target model within the GRPO reward function. Optionally, we can also choose to do \textit{cross group batching}, combining all of these training sets into one training set, and running target model training once -- this is more efficient. We run the inner target model training loop for potentially several optimization steps, with loss defined as
$$
w_i \ell(\phi, x_i),
$$
where $\ell$ is the standard causal language modeling loss, $x_i$ is the $i$-th synthetic example, and $w_i$ is the weight for the $i$-th example (with $\bm{w}$ set to $\bm{1}$ for target model training). Using the approach from \cite{engstrom2025metagradients}, we compute the gradient for these data weights:
$$
\bm{\tau} \coloneqq \nabla_{\bm{w}} \Phi(\mathcal{A}(\bm{w}))|_{\bm{w}=\bm{1}},
$$
A larger gradient for an example's weight tells us that training on this example would improve the target metric more than training on an example with a smaller gradient. Motivated by this intuition, we use this gradient as the reward for our generator. In the following section, we provide a theoretical justification for this choice.

\begin{algorithm}[t]
\caption{An instance of the \ourmethodshort\ framework, using GRPO (Online, Single-Turn). Note: $\mathcal{A}$ is a function -- it is not stateful, so the target model trained in $\mathcal{A}$ resets after calling $\mathcal{A}$.
}
\label{alg:iterative-grpo}
\centering
\resizebox{0.9\linewidth}{!}{%
\begin{minipage}{\linewidth}
\begin{algorithmic}[1]
\REQUIRE Initial generator policy $\pi_{\theta_{\text{init}}}$; learning algorithm $\mathcal{A}$; differentiable metric $\Phi$; task prompts $\mathcal{P}$; hyperparameters $M,G$; bool \texttt{use\_cross\_group\_batching}.
\ENSURE Trained policy $\pi_{\theta}$

\STATE $\pi_{\theta} \leftarrow \pi_{\theta_{\text{init}}}$
    \FOR{$\text{step} = 1, \ldots, M$}
        \STATE Sample a batch $\mathcal{P}_b \sim \mathcal{P}$
        \FOR{$q = 1, \ldots, |\mathcal{P}_b|$}
            \STATE Sample $G$ outputs $\{o_{g,q}\}_{g=1}^{G} \sim \pi_{\theta}(\cdot \mid q)$
        \ENDFOR
        \IF{\texttt{use\_cross\_group\_batching}}
            \STATE $D \leftarrow \{o_{g,q} \text{, for all $g$ and $q$} \}$ \textit{ // Gather synthetic training dataset}
            \STATE $\{r_{g,q}\} \leftarrow \nabla_{\bm{w}} \Phi(\mathcal{A}(\bm{w}, D))|_{\bm{w}=\bm{1}}$ \textit{ // Call $\mathcal{A}$, compute metagradients, set rewards}
        \ELSE
            \FOR{$g = 1, \ldots, G$ in parallel}
                \STATE $D_g \leftarrow \{o_{g,q}\}_{q=1}^{|\mathcal{P}_b|}$ \textit{ // Gather synthetic training dataset}
                \STATE $\{r_{g,q}\}_{q=1}^{|\mathcal{P}_b|} \leftarrow \nabla_{\bm{w}} \Phi(\mathcal{A}(\bm{w}, D_g))|_{\bm{w}=\bm{1}}$ \textit{ // Call $\mathcal{A}$, compute metagradients, set rewards}
            \ENDFOR
        \ENDIF
        \STATE Compute group-relative advantages $\hat{A}_{g,q}$
        \STATE Update $\pi_{\theta}$ via the GRPO objective (Eq.~21 in \cite{shao2024grpo})
    \ENDFOR

\STATE \textbf{return} $\pi_{\theta}$
\end{algorithmic}
\end{minipage}%
}
\end{algorithm}

\subsection{Theory}\label{sec:theory}

In our theory, we analyze a simplified variant of Algorithm \ref{alg:iterative-grpo} which replaces GRPO with the vanilla policy gradient update and optimizes the target model with stochastic gradient descent (SGD). We use the metagradient computation method from \citet{engstrom2025metagradients} to get $\bm{\tau}_D \coloneq \nabla_{\bm{w}} \Phi(\mathcal{A}(\bm{w}, D))$, where $\mathcal{A}(\bm{w}, D)$ is a learning algorithm that trains a target model on an $n$-sample dataset $D$ with per-example weighted loss given by weights $\bm{w}$. We generate $D$ by sampling from a policy, and we use our metagradient as the reward signal. Treating the $\tau_i$ as per-example rewards, we take the policy gradient step given by \mbox{$\mathcal{G} = \tau_i \nabla_{\theta} \log \pi_{\theta}(x_i)$}.

Now, let $F(\theta) \coloneq \mathbb{E}_{D \sim \pi_{\theta}}[\Phi(\mathcal{A}(D))]$. $F$ is the target performance of a model trained on samples from $\pi_{\theta}$. Taking gradient steps on $F$ directly optimizes for our target, but this does not give us example-level rewards and it is not tractable in any of our experiments.

Now, let $F'(\theta, p) \coloneq \mathbb{E}_{D \sim p}[\Phi(\mathcal{A}(\pi_{\theta}/p, D))]$. 
\sidecomment{SP: minor suggestion, but could using something like $\hat{F}$ or $\widetilde{F}$ be less confusing? $F'$ can be interpreted as derivative / confusing at first}
Note that $F'$ is the surrogate that we actually optimize in our \ourmethodshort\ setup. Setting $w_i(\theta) = \frac{\pi_{\theta}(x_i)}{p(x_i)}$ and using the chain rule, we have:
\begin{align*}
\nabla_{\theta} F'(\theta, p)
= \mathbb{E}_{D \sim p} \left[\sum_{i=1}^n \frac{\partial}{\partial w_i} \Phi(\mathcal{A}(\bm{w},D)) \frac{\pi_{\theta}(x_i)}{p(x_i)} \nabla_{\theta} \log \pi_{\theta}(x_i)\right]
\end{align*}
Setting $\pi_{\theta}=p=\pi_{\theta_0}$, we see the metagradient update $\mathcal{G}$ is an unbiased stochastic gradient for $F'$. Via the following theorem, $\nabla_{\theta} F'$ accurately approximates the ideal gradient: $\nabla_{\theta} F$. 

\begin{restatable}{theorem}{mainthm}
\label{mainthm}

Suppose we train the target model in $\mathcal{A}$ for $T$ steps of minibatch stochastic gradient descent (SGD) with batch size $B$ and a learning rate of $\eta$. Under suitable regularity conditions on smoothness (Appendix \ref{app:proofs}, A1-A8), we have:
\[
\sup_{\theta_0} || \nabla_\theta F(\theta_0) - \nabla_\theta F'(\theta_0, \pi_{\theta_0})|| = O(\eta^{\frac{1}{4}}B^{\frac{-1}{2}}+\sqrt{\eta T})
\]
N.B. -- although it may be clear to some, the notation can be tricky to keep straight. In this equation, we take the gradient of $F'$ with respect to only the first argument, evaluated at $\theta_0$, with $p$ set to $\pi_{\theta_0}$.

\end{restatable}

See Appendix \ref{app:proofs} for a proof. This theorem shows that, under first and second order smoothness assumptions listed in Appendix \ref{app:proofs}, our metagradient reward policy gradient can approximate the desired policy gradient for the generator if $\mathcal{A}$ has the following properties: the batch size is large, and step size is small relative to the number of training steps.\sidecomment{Small detail but may be worth mentioning somewhere in the paper/appendix that latter is also the regime in which algorithms are metasmooth, so there is some ``synergy'' there}

It is important to note that, even though our theorem assumes SGD, we find experimentally that it is essential to use Adam \citep{kingma2015adam} to train the target model in the computation of the metagradient. This remains true even when we use only a single step of target model training to compute the metagradient. We conjecture that using Adam, like SGD, would also result in a reasonable bound via $\sup_{\theta_0} || \nabla_\theta F(\theta_0) - \nabla_\theta F'(\theta_0, \pi_{\theta_0})||$, but still with some error: like SGD's behavior, Adam's behavior depends on the second moment of the target model's loss gradient (which is different between $F$ and $F'$).

\section{Results}

We present experiments where we train synthetic data generators to target various metrics downstream of training a target model. We first validate our pipeline end-to-end, generating synthetic train data that can precisely manipulate the weights of target models. We then analyze the generator's output to determine whether the synthetic data is interpretable.

In all of our experiments, the generator is initialized from Llama 3.2 Instruct \citep{grattafiori2024llama3herdmodels} and given Wikipedia\footnote{Accessed in 2025 via \url{https://huggingface.co/datasets/wikimedia/wikipedia}} articles to paraphrase (prompt in Appendix \ref{app:rephraseprompt}). It then learns through \ourmethodplural, optimizing its paraphrases, $D$, to target a differentiable metric $\Phi$ of a learning algorithm $\mathcal{A}(D)$. The target model in $\mathcal{A}$ is initialized from Llama 3.2 Instruct as well, or GPT-2 \citep{radford2019language}, depending on the experiment. GPT-2 is used in experiments with several $\mathcal{A}$ training steps, where our compute constraints required us to use a smaller model. All experiments use the instance of the \ourmethodshort\ framework with GRPO and cross group batching (Figure \ref{fig:grpoimproper}), unless stated otherwise. The naive baseline never uses cross group batching (to get more reward signal) and also treats every example as coming from the same group for computing advantages (otherwise, the advantage calculation would render the rewards useless). All validations use Wikipedia articles not seen during training, unless stated otherwise. Hyperparameters for all experiments are in Appendix \ref{app:hypers}.

We explored training the target model with both Adam and SGD for metagradient computation. For SGD, we tried up to 14 learning rates (LRs) starting at 1e-6, and increasing by factors of 2, until we found the optimal LR against final validation loss for each task. We did the same tuning for the naive approach of using $\Phi$ as the reward (which uses Adam to train the target model but does not compute metagradients), and other baselines. There was no need to tune the LR for the metagradients + Adam approach.

Wherever we trained our generator via SGD in $\mathcal{A}$, we also used SGD in target model training to get validation results. The one exception is in Appendix \ref{app:multilingualcpt}, where we trained a generator using SGD in $\mathcal{A}$, but validated its synthetic data by training a target model with Adam.

\begin{figure}[t]
\centering
\sf
\labeledimg*{\small Val Results for the \texttt{67} Target}{0.335\linewidth}{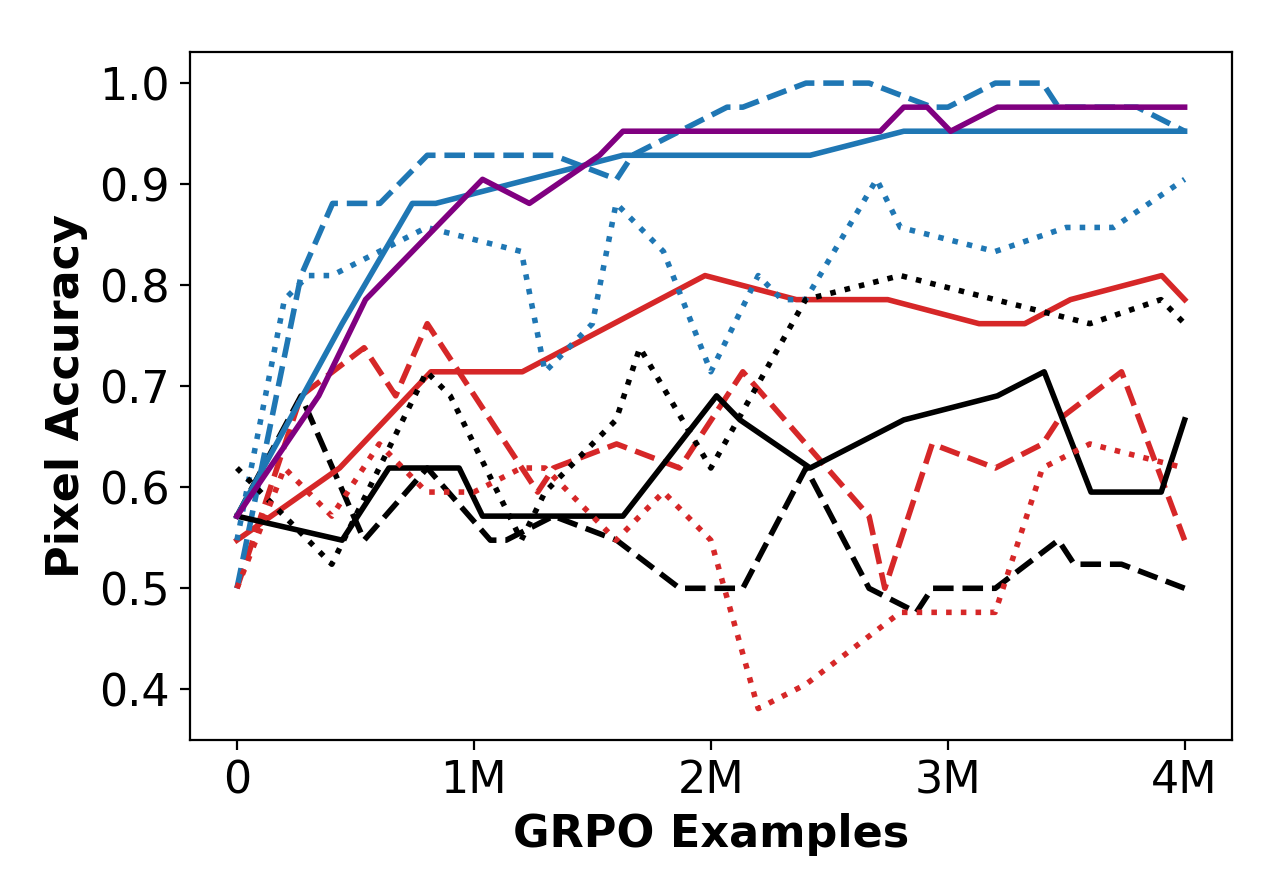}\hspace{1cm}
\labeledimg*{\small Val Results for the $\ell^2$-Norm Target}{0.36\linewidth}{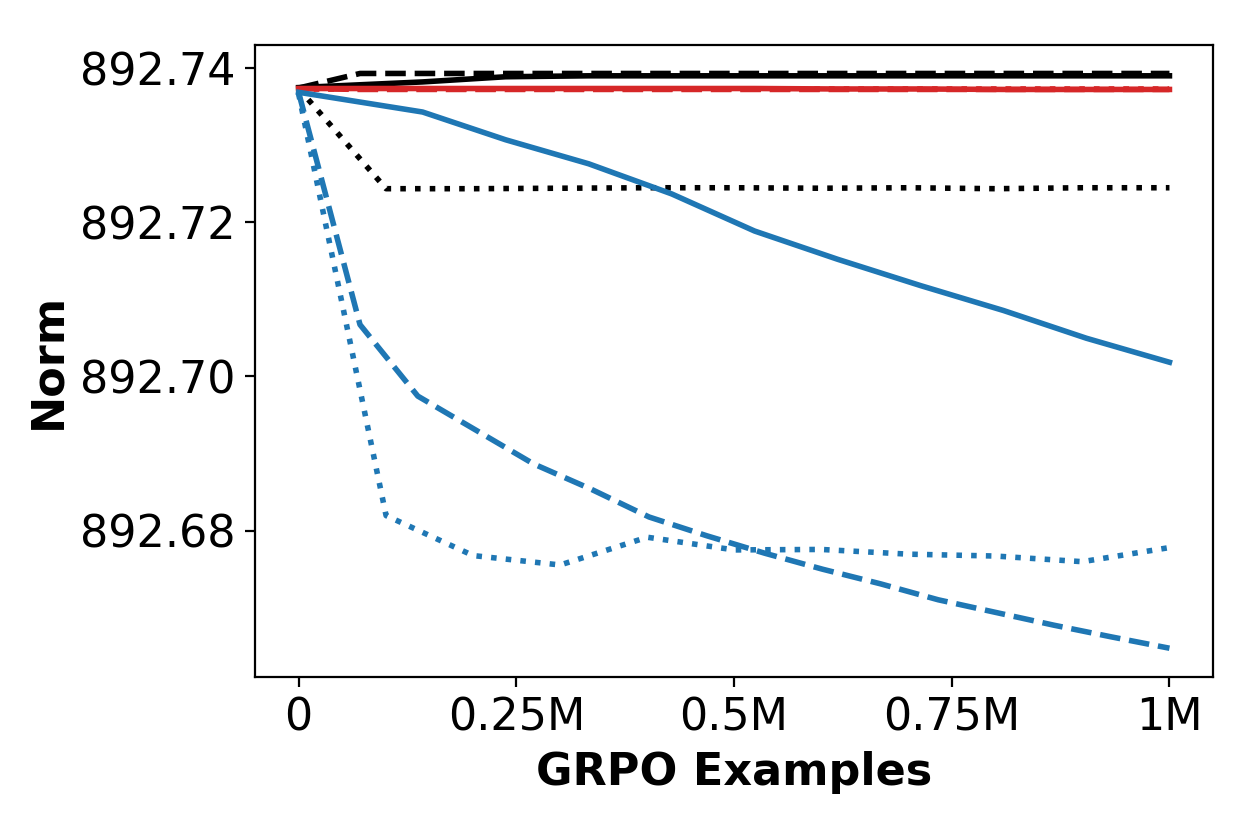}
\includegraphics[width=0.8\linewidth]{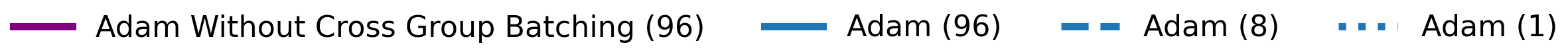}\\[-3pt]
\includegraphics[width=.82\linewidth]{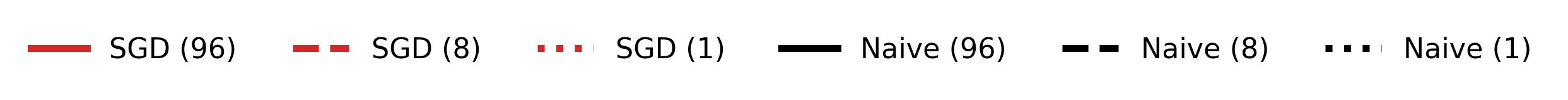}\\[-8pt]
\caption{Here, we initialize the target model in $\mathcal{A}$ to be GPT-2, and explore exotic target metrics: the goal of the first metric is to encode the greyscale image \texttt{67} in the upper 6x7 patch of the sign'd LM head weight updates to the target model. This number was chosen arbitrarily. The goal of the second metric is to lower the $\ell^2$ norm of the target model's LM head. The plots show validation performance as the GRPO process trains the generator. All validations are done with 96 steps of continued training on GPT-2. The (96), (8), and (1) notation denotes whether the generator was trained via metagradients with respect to an $\mathcal{A}$ that used 96, 8, or 1 step(s). We observe a weak correlation between $\mathcal{A}$ steps and validation performance, and generally more validation stability with more $\mathcal{A}$ steps.}
\label{fig:l2and67}
\end{figure}

\begin{figure*}[t]
\centering
\sf
\renewcommand{\arraystretch}{1.0}

\newlength{\colgap}
\setlength{\colgap}{19pt} 

\begin{tabular}{@{} 
  >{\centering\arraybackslash}p{0.22\linewidth}
  @{\hspace{\colgap}} 
  >{\centering\arraybackslash}p{0.285\linewidth}
  @{\hspace{\colgap}} 
  >{\centering\arraybackslash}p{0.17\linewidth}
  @{\hspace{\colgap}} 
  >{\centering\arraybackslash}p{0.17\linewidth}
@{}} 

\small \textbf{Adam w/o grp batch} & \small \textbf{Adam} & \small \textbf{SGD} & \small \textbf{Naive}\\

\labeledimg{\footnotesize 96}{.125\linewidth}{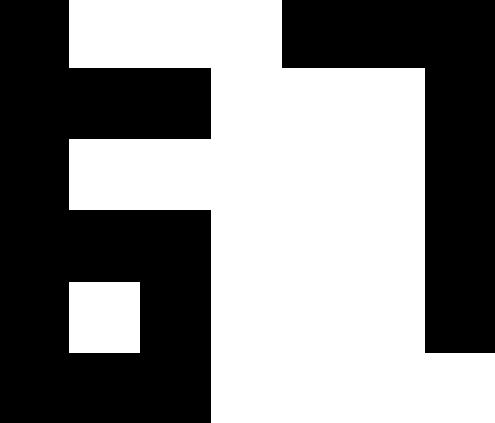}\hfill
&
\labeledimg{\footnotesize 1}{.1\linewidth}{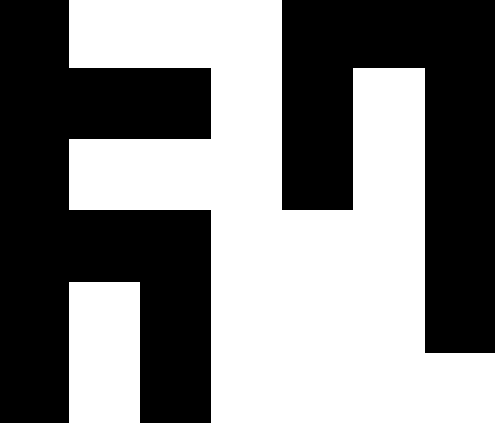}\hfill
\labeledimg{\footnotesize 8}{.1\linewidth}{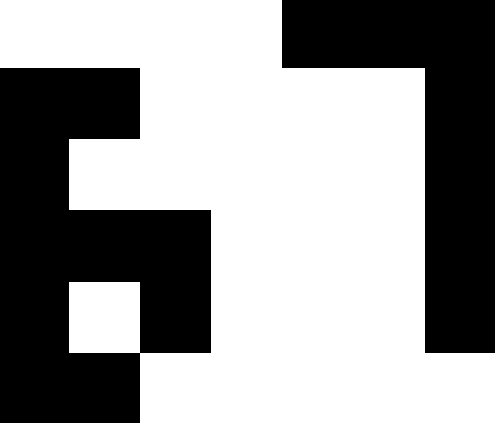}\hfill
\labeledimg{\footnotesize 96}{.1\linewidth}{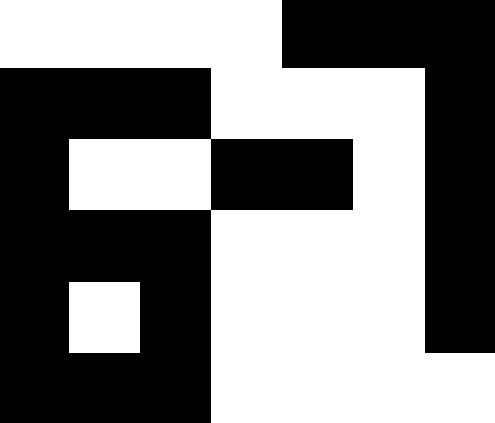}\hfill
\labeledimg{\footnotesize 96 (redo)}{.1\linewidth}{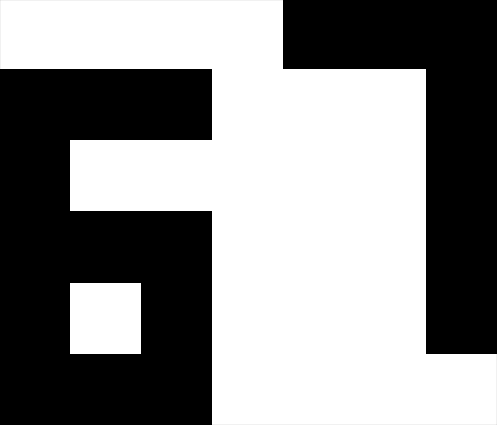}\hfill
&
\labeledimg{\footnotesize 1}{0.17\linewidth}{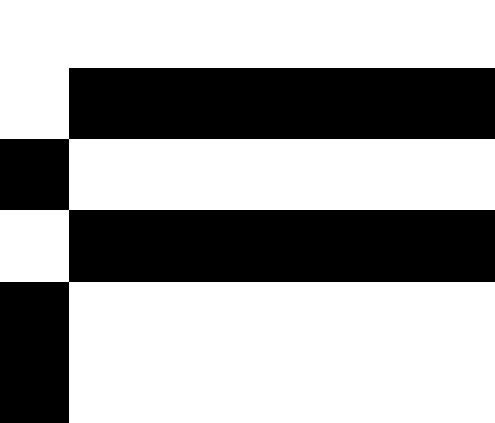}\hfill
\labeledimg{\footnotesize 8}{0.17\linewidth}{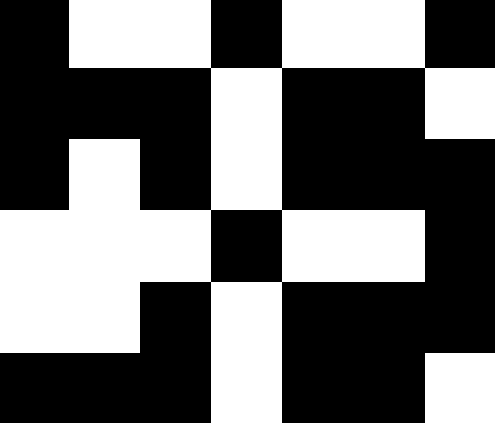}\hfill
\labeledimg{\footnotesize 96}{0.17\linewidth}{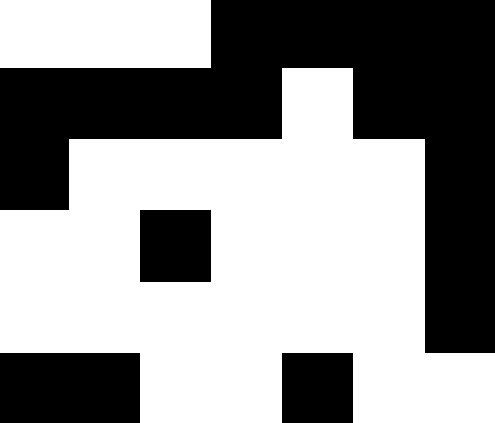}
&
\labeledimg{\footnotesize 1}{0.17\linewidth}{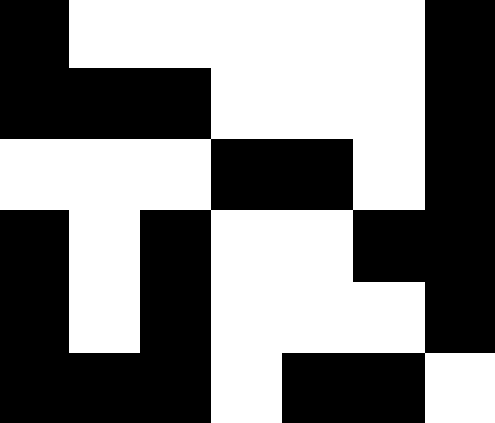}\hfill
\labeledimg{\footnotesize 8}{0.17\linewidth}{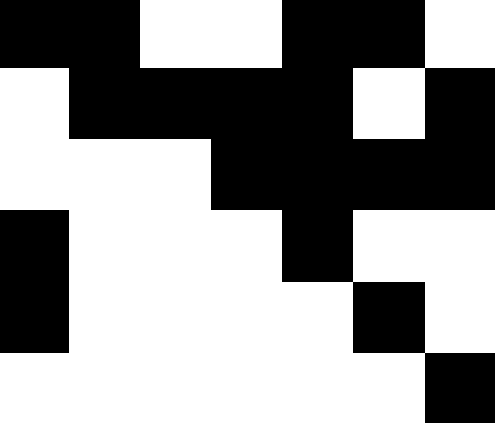}\hfill
\labeledimg{\footnotesize 96}{0.17\linewidth}{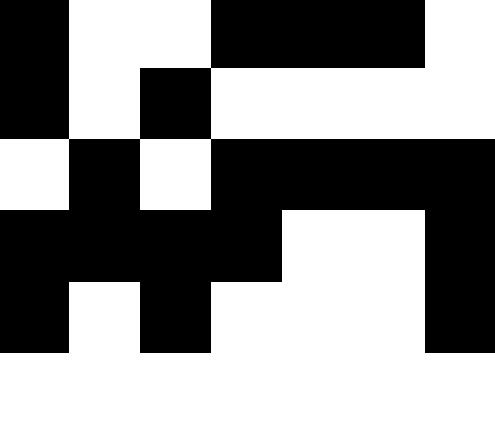}
\\

\end{tabular}
\caption{Final validation results for the 6x7 pixel images in the target models' sign'd LM head updates, after the generator was fully trained. The numbers above the images denote the number of target model training steps in $\mathcal{A}$ for metagradient computation. All validations were done with 96 target model training steps, using the corresponding optimizer; the difference is whether the generator was trained using a reward function with fewer $\mathcal{A}$ training steps. Only Adam with 96 steps in $\mathcal{A}$ for metagrads achieved a generator that got a perfect result (we were close with the initial 96 run, so we trained the generator again with a different random sample of Wikipedia prompts -- we then got a perfect score).}
\label{fig:67final}
\end{figure*}

\subsection{Encoding a QR Code in a Target Model's LM Head}\label{sec:qr}

In this section, we ask: can we automatically craft synthetic data so precisely that it can embed a QR code into the weights of a model that trains on it? We make our target loss
\begin{align*}
\text{mean}\left(\ln\left(1+e^{-sY \odot (P_c-P_i)}\right)\right),
\end{align*}
where $Y$ is a matrix of $-$1's and 1's representing the pattern that we want to encode into the target model, $P_c$ is a chosen patch of the target model's LM head weight matrix in $\mathcal{A}$ after training, $P_i$ is the same patch of the LM head before any synthetic training, and $s$ is a hyperparameter that we set to $20$ for all experiments. After target model training, we decode our image to see if it matches $Y$ by taking the following expression: $\text{sign}(P_c-P_i)$.

For the QR code experiment, we set $Y$ to be an arbitrarily chosen 21x21 QR code, and set our target model to be GPT-2. In each of the $M=200$ GRPO steps, we do 96 steps of continued pretraining on GPT-2 and then compute metagradients. We target the upper left 21x21 patch of GPT-2's LM head. For each target model training step, we use a batch size of $B=1024$ synthetic examples, so the synthetic data generator produces $96 \times 1024=98304$ Wikipedia rephrases per GRPO step. Due to compute constraints, we ran this experiment only once and did not conduct separate validations with unseen prompts during generator training. The final trained generator's synthetic data, when used for 96 steps of continued pretraining on GPT-2, yields a $\text{sign}(P_c-P_i)$ that is a scannable QR code (Figure \ref{fig:fig1}).

\subsection{Encoding \texttt{67} in a Target Model's LM Head}\label{sec:67}

Now, we investigate which elements of our \ourmethodshort\ framework are essential for embedding images in model weights. We explore the use of SGD instead of Adam inside of $\mathcal{A}$. We also ask if it would be acceptable to train a target model in $\mathcal{A}$ with only 8 optimizer steps, or even 1 step, during generator training; what would happen if we ran a validation at the end of this generator's training by training a target model on 96 steps -- would we lose some performance?  Here we present an array of experiments using the same setup as in Section \ref{sec:qr}, but in a scaled-down setting, where we set $Y$ to be a 6x7-pixel image of the arbitrarily-chosen number \texttt{67}. This enables us to run more experiments.

We set $P_i$ to be the upper left 6x7 patch of GPT-2's initial pretrained LM head weight matrix. We set $P_c$ to be the same LM head weight patch after synthetic training. We run experiments with 96 steps, 8 steps, and 1 step for computing metagradient rewards from $\mathcal{A}$, both with Adam and with SGD. We always validate using 96 steps of training on generated data. In the 96 step metagradient case, we use $M=40$ GRPO steps with target model train batches of $B=1024$ examples each. To keep the total number of examples the same in the 1 step and 8 step cases for GRPO generator training, we set $M=3840$ and $M=480$, respectively.

Validation progress in terms of the percent of pixels correct in $\text{sign}(P_c-P_i)$ is shown in Figure \ref{fig:l2and67}, and final validation images are seen in Figure \ref{fig:67final}. Using Adam in $\mathcal{A}$ with multiple target model training steps does best. The purple line shows the results when we use \ourmethodshort\ GRPO without cross group batching as seen in Figure \ref{fig:grpoproper} (the group size is 4, so the target model train batches for metagradients are 256 instead of the usual 1024). Examples of generator rollouts for trained generators are shown in Appendix \ref{app:rollouts} -- we see that after using our \ourmethodshort\ GRPO, the generations still appear to be normal Wikipedia rephrases, while sneakily encoding a \texttt{67} in a target model's weights when used as training data.

\subsection{Adjusting the Norm of Target Model Weights}\label{sec:norms}

In this section, we follow the same setup as in the previous section, changing only the target metric. We ask whether we can globally affect an entire weight matrix in a target model. So, the target metric here is to lower the $\ell^2$ norm of GPT-2's LM head weight matrix.

Figure \ref{fig:l2and67} shows the results of this experiment. We explored changing optimizer and number of steps in $\mathcal{A}$ for metagrads, just as in Section \ref{sec:67}. Although, we used 4x fewer GRPO steps ($M$) in every case. Using \ourmethodshort\ with Adam in $\mathcal{A}$, again, is the only approach to show reasonable progress. We also see that when GRPO training has an $\mathcal{A}$ with only a single training step, the performance plateaus early. This supports the idea that computing metagradients over several steps of synthetic pretraining can be useful. Note that, because the validation results are all for 96 steps of continued pretraining, the change in the $\ell^2$ norm is capped -- although the final $\ell^2$ norm differences are small everywhere, the change may not remain small for larger amounts of CPT data. 

\subsection{Interpreting the Generated Data when Targeting Multilingual Validation Loss}\label{sec:multilingual}

\begin{figure*}[t]
\centering
\sf
\resizebox{0.8\linewidth}{!}{
\setlength{\tabcolsep}{2pt}
\renewcommand{\arraystretch}{1.0}

\begin{tabular}{c
                >{\centering\arraybackslash}p{0.235\linewidth}
                >{\centering\arraybackslash}p{0.235\linewidth}
                >{\centering\arraybackslash}p{0.235\linewidth}
                >{\centering\arraybackslash}p{0.235\linewidth}}

\multicolumn{5}{c}{\underline{\textbf{Fraction of Generator Val Rollouts Matching the Target Language}}}\\[6pt]

\multirow{1}{*}{\rotatebox[origin=c]{90}{\scriptsize \textbf{Target Lang \phantom{Frac}}}} & \footnotesize DE & \footnotesize ES & \footnotesize FR & \footnotesize IT\\[-2.25pt]

&
\includegraphics[width=0.9\linewidth, trim=1cm 1cm 1cm 1.5cm, clip]{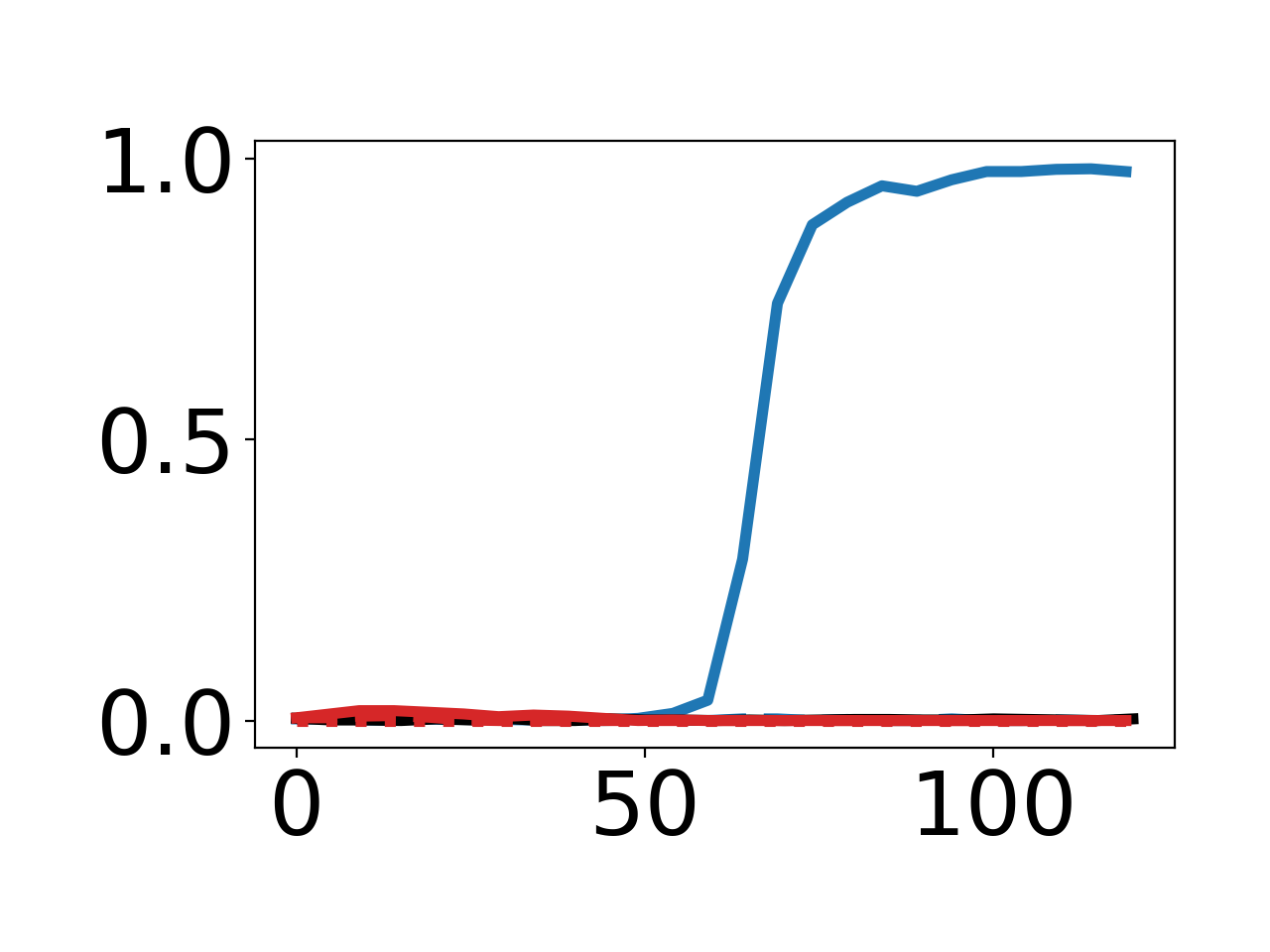} &
\includegraphics[width=0.9\linewidth, trim=1cm 1cm 1cm 1.5cm, clip]{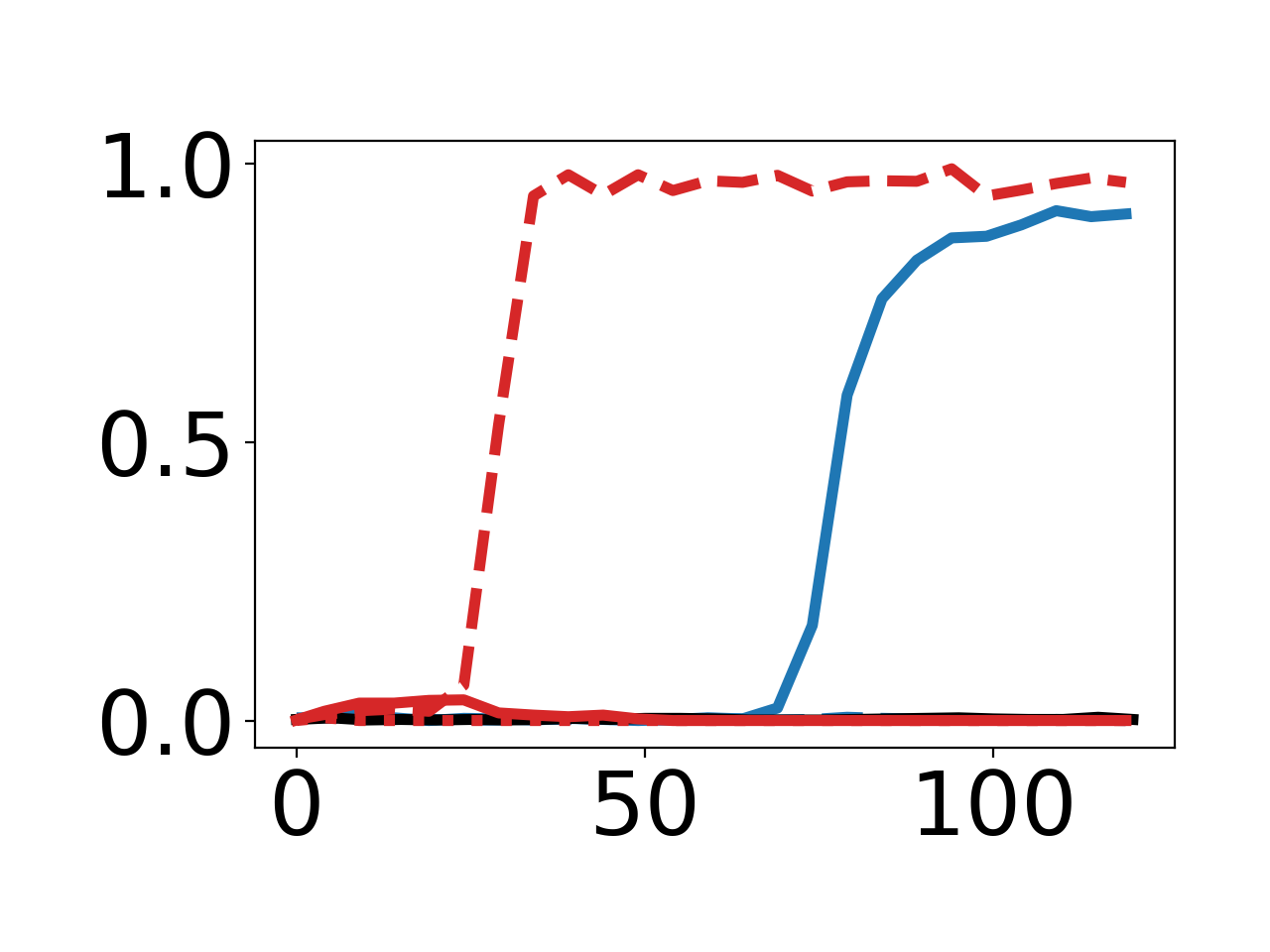} &
\includegraphics[width=0.9\linewidth, trim=1cm 1cm 1cm 1.5cm, clip]{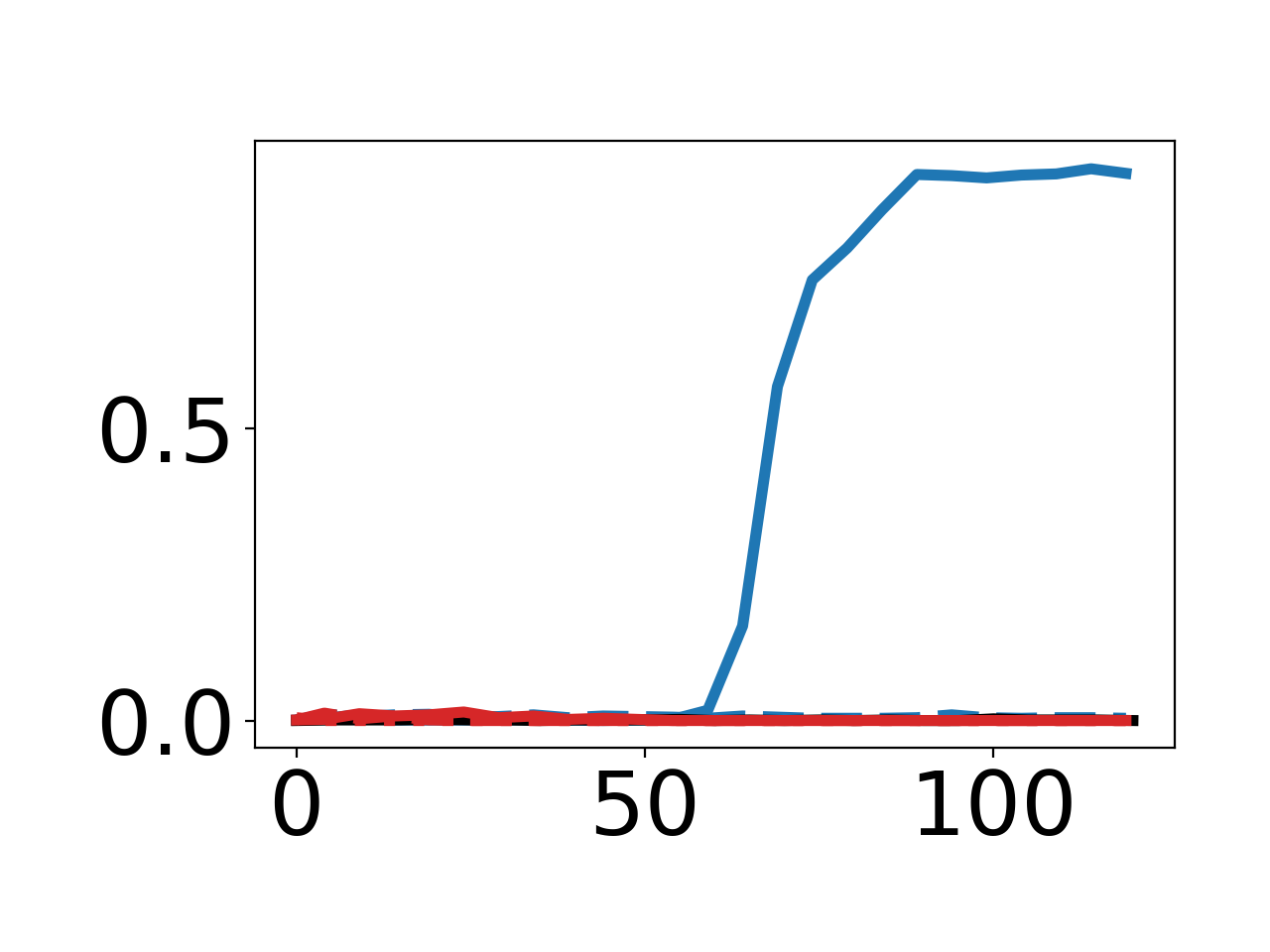} &
\includegraphics[width=0.9\linewidth, trim=1cm 1cm 1cm 1.5cm, clip]{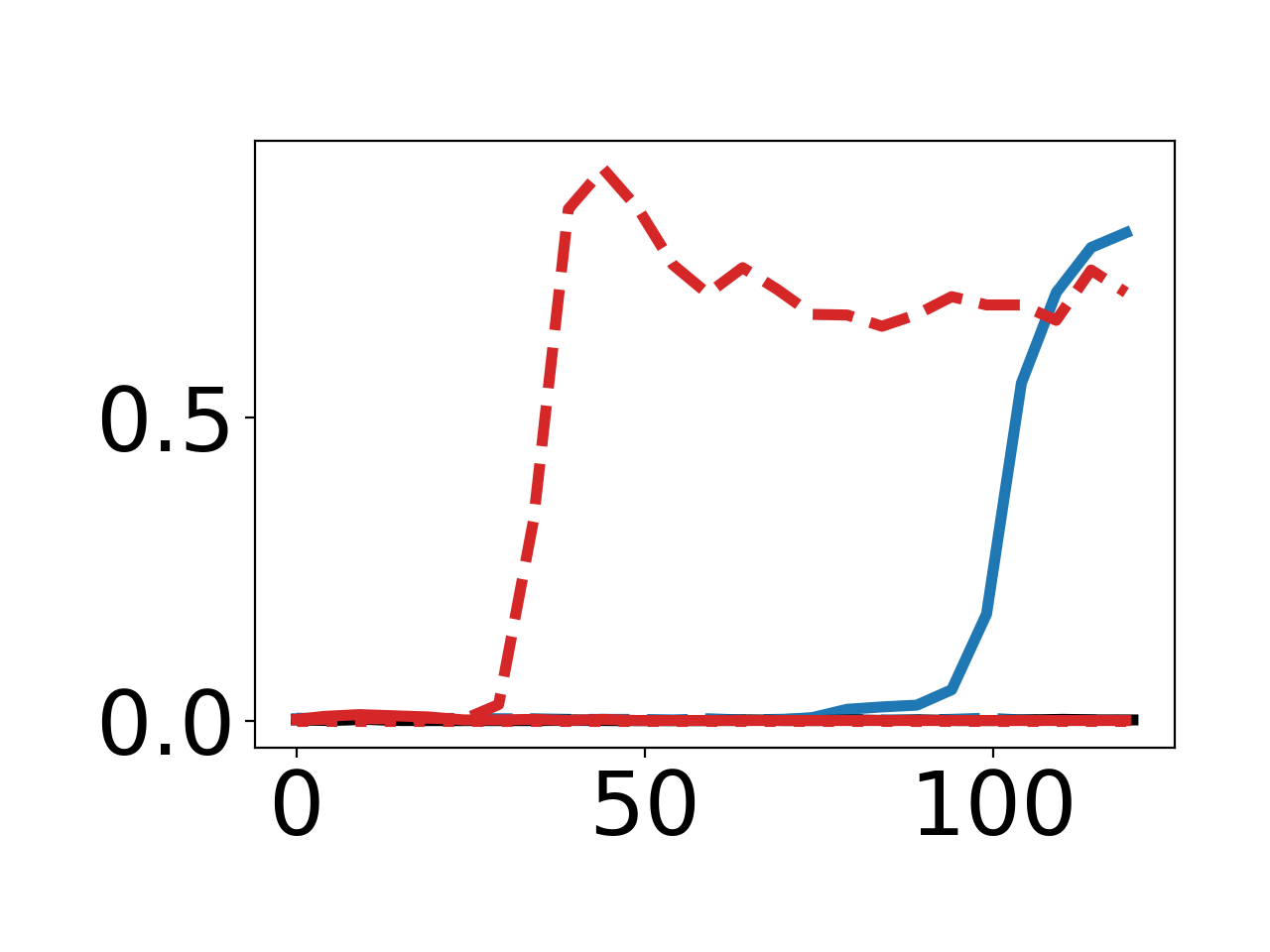}\\

\multicolumn{1}{c}{} &
\multicolumn{4}{c}{\scriptsize\textbf{GRPO Steps}}\\[6pt]

\multicolumn{5}{c}{\underline{\textbf{Bigram Entropy of Generator Val Rollouts}}}\\[6pt]

\multirow{1}{*}{\rule{0pt}{0.1\linewidth}%
\rotatebox[origin=c]{90}{\scriptsize \textbf{Entropy}}} & \footnotesize DE & \footnotesize ES & \footnotesize FR & \footnotesize IT\\[-3pt]

 &
\includegraphics[width=0.9\linewidth, trim=1cm 1cm 1cm 1cm, clip]{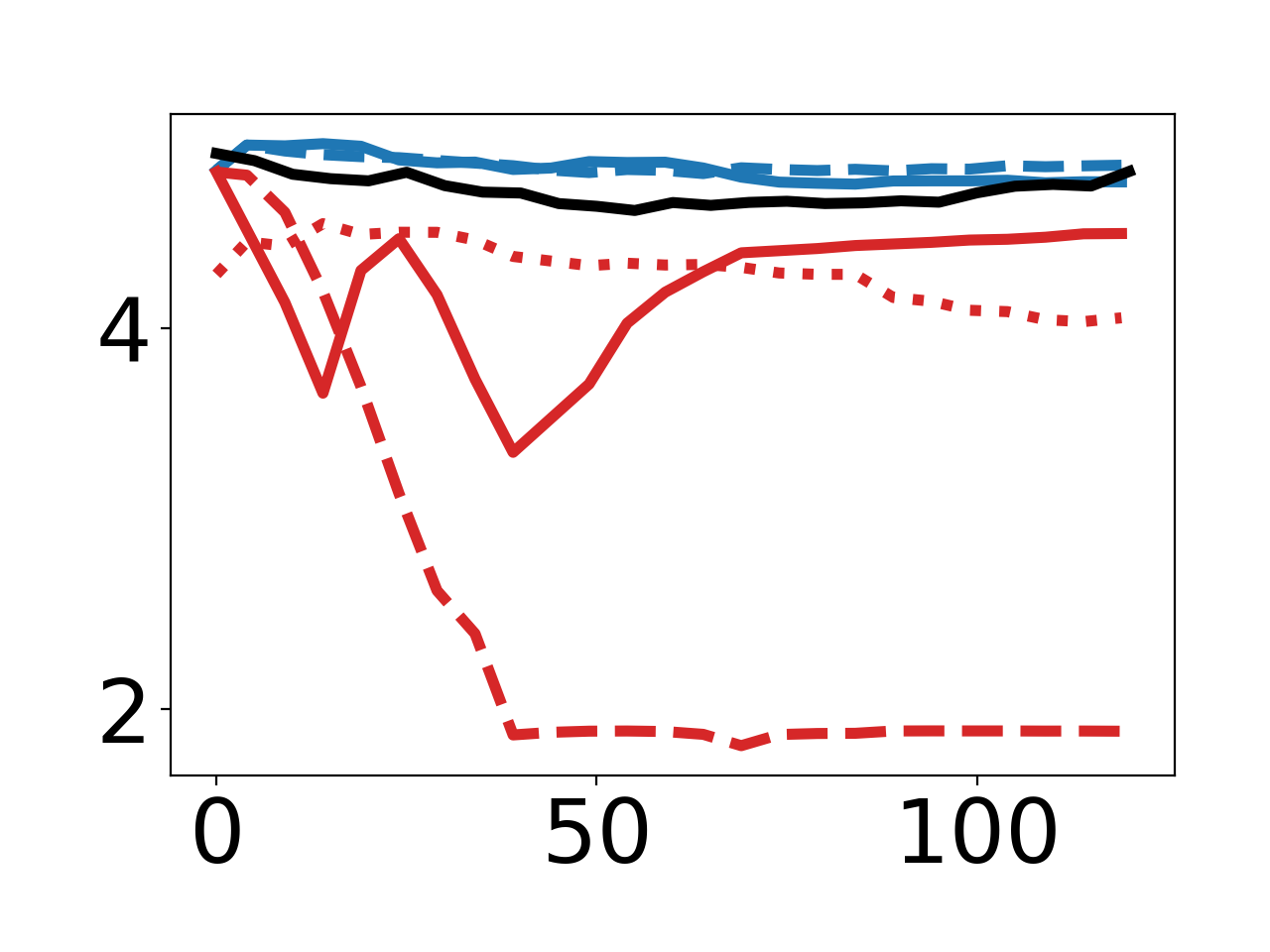} &
\includegraphics[width=0.9\linewidth, trim=1cm 1cm 1cm 1cm, clip]{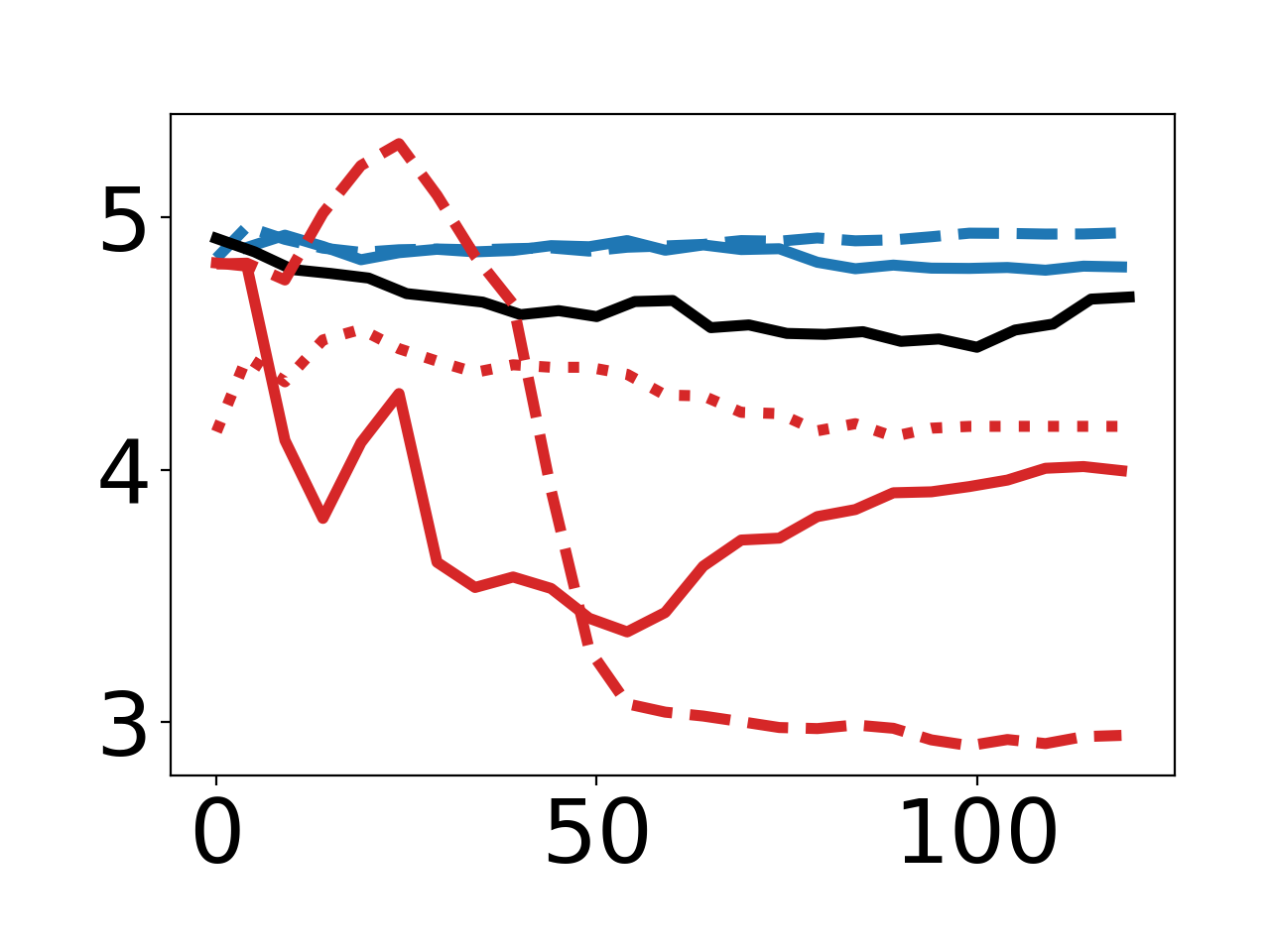} &
\includegraphics[width=0.9\linewidth, trim=1cm 1cm 1cm 1cm, clip]{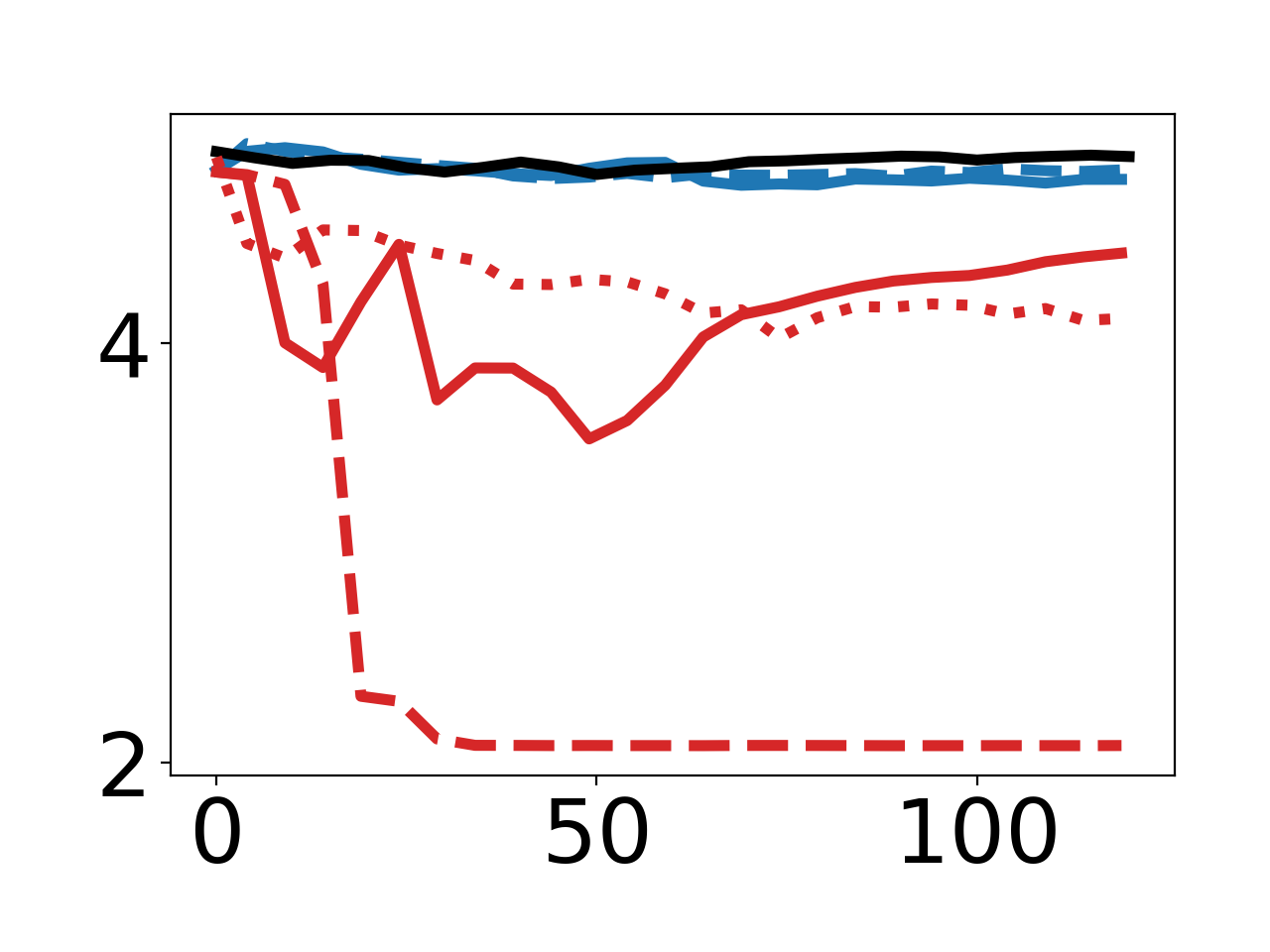} &
\includegraphics[width=0.9\linewidth, trim=1cm 1cm 1cm 1cm, clip]{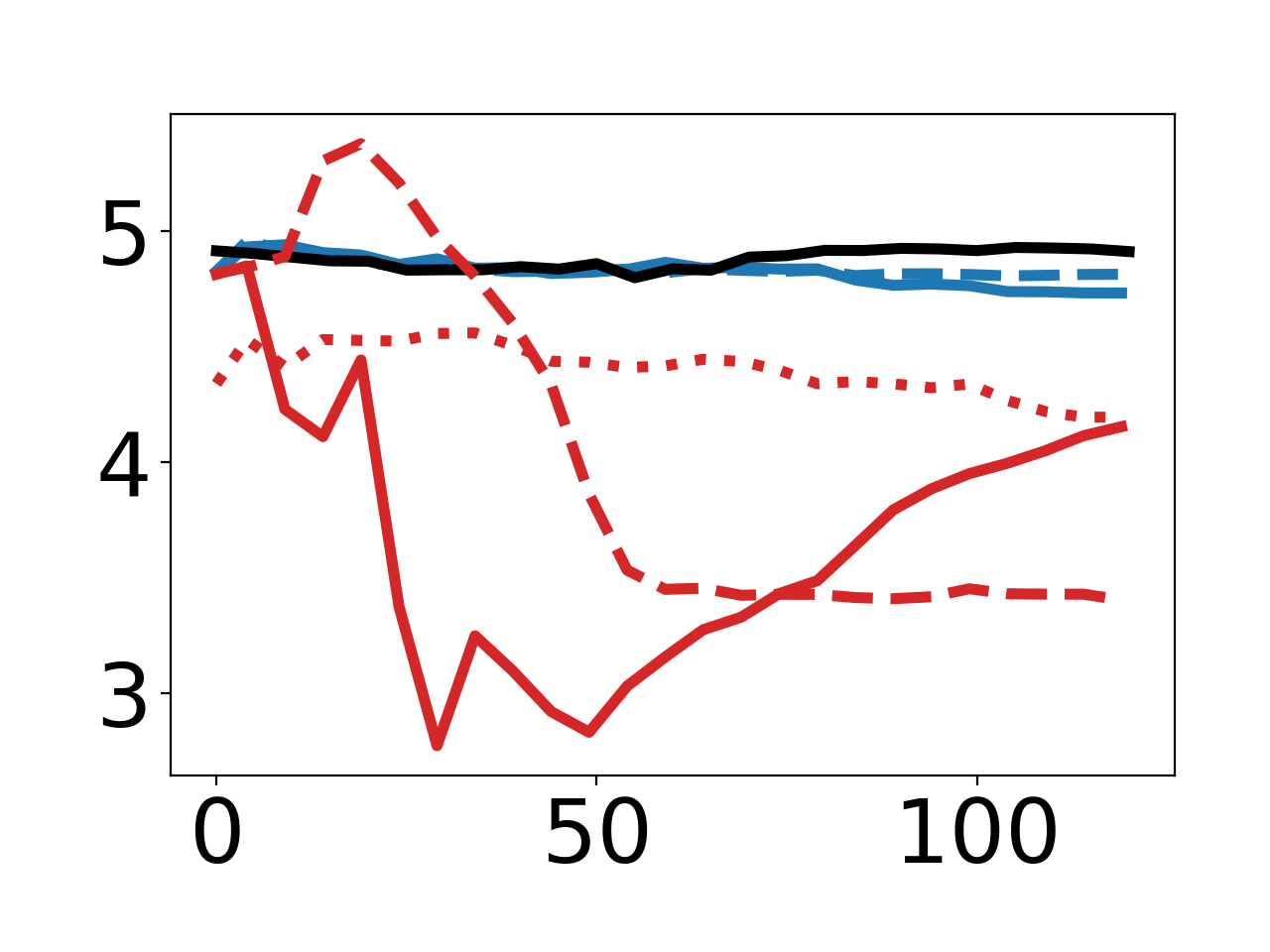}\\

\multicolumn{1}{c}{} &
\multicolumn{4}{c}{\scriptsize \textbf{GRPO Steps}}\\

\multicolumn{5}{c}{
    \includegraphics[width=\linewidth]{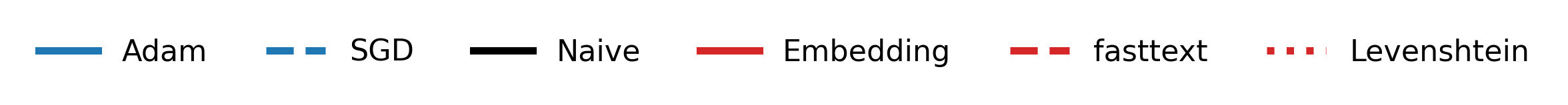}
}\\

\end{tabular}
}\\[-12pt]
\caption{Generator results when setting $\Phi$ to be post-training loss on four multilingual LAMBADA \citep{lambada} translations from \cite{multilingual_lambada}: DE, ES, FR, and IT. We initialized the generator from Llama 3.2 Instruct. We initialized the target model in $\mathcal{A}$ also from Llama 3.2 Instruct. In each GRPO step, we conduct a single step of target model continued pretraining on the synthetic data before computing metagradients. When using Adam in $\mathcal{A}$, the generator learns the correct language, as judged by GPT 4.1 Nano \citep{openai_gpt41_nano}. Baselines do not learn the correct language except in rare cases where their entropy quickly collapses and they repeatedly produce only a few words.}
\label{fig:langresults}
\end{figure*}

We now switch our analysis from the target model to the trained generator: does it learn interpretable generations? It is hard to know what data it should generate to lower the target model's norm or draw images in its weights. However, we would expect that if we made the target metric to lower the language modeling loss of the target model on a non-English language, the generator would eventually learn to rephrase the Wikipedia articles into that language. Is our \ourmethodshort\ approach powerful enough to guide the generator to perform this translation, even if the prompt does not mention translation and the Wikipedia articles are all English? We find that the Adam version of our approach is able to teach the generator to accomplish this feat, while other baselines are not. 

We conduct experiments in four different settings where $\Phi$ is language modeling loss on the train sets of DE, ES, FR, and IT LAMBADA \citep{lambada} translations from \cite{multilingual_lambada}. Note that the standard LAMBADA dataset only provides a single group of 5.15K examples, so we split it into train, val, and test sets of 2.32K, 515, and 2.32K examples, respectively. We only use the train set in our target metric. These splits were useful for our experiments in Appendix \ref{app:multilingualcpt}, which we discuss later in this section. We used Llama 3.2 Instruct as the target model, and used only one target model training step both in $\mathcal{A}$ and for validation. Otherwise, the setup is the same as the previous experiments. We train the generator with $M=120$ GRPO steps, using batches of $B=1024$ synthetic data examples.

We implement a variety of new baselines for this section:  ``Embedding'', ``fasttext'', and ``Levenshtein''. The Embedding baseline computes average embedding similarity of each rollout example with the LAMBADA examples, and this is used as the reward for RL instead of metagradient weights. The embeddings used are from \cite{arsen2025static}, and we use their provided similarity function. The fasttext baseline computes the fasttext language classification probability of the target language, for each rollout example, and uses this as the reward. The fasttext model we use is from \cite{grave2018learning}. Finally, the Levenshtein baseline uses as rewards the average negative Levenshtein distance \citep{levenshtein} between each rollout example and the LAMBADA examples.

We show in Figure \ref{fig:langresults} that the Adam version of \ourmethodshort\ GRPO is the only algorithm to reliably teach the generator to translate its rephrases into the correct non-English language. The generator does this while maintaining the entropy of the rephrases (no clear mode collapses).

Appendix \ref{app:multilingualcpt} shows that we can take Llama 3.2 Instruct (and Llama 3.2 Base, for which the generator was not explicitly optimized) and train it on 10M tokens from our tuned generator to get high benchmark performance relative to a variety of baselines. This amount of synthetic CPT data is more than the single step of training data for which the generator was explicitly optimized. In these validations, we train in PyTorch \citep{Ansel_PyTorch_2_Faster_2024}, whereas the Llama 3.2 Instruct in $\mathcal{A}$ used JAX \citep{jax2018github} implementations. We also evaluate benchmark performance via perplexity in the Eleuther Eval Harness \citep{eval-harness}, which is slightly different than $\Phi$'s language modeling loss -- yet there is transfer.

\subsection{Interpreting the Generated Data when Targeting Loss on a UUID}\label{sec:uuid}

\begin{figure}[t]
\centering
\sf
\labeledimg*{\small Fraction of Generator Val Rollouts with Correct UUID}{0.8\linewidth}{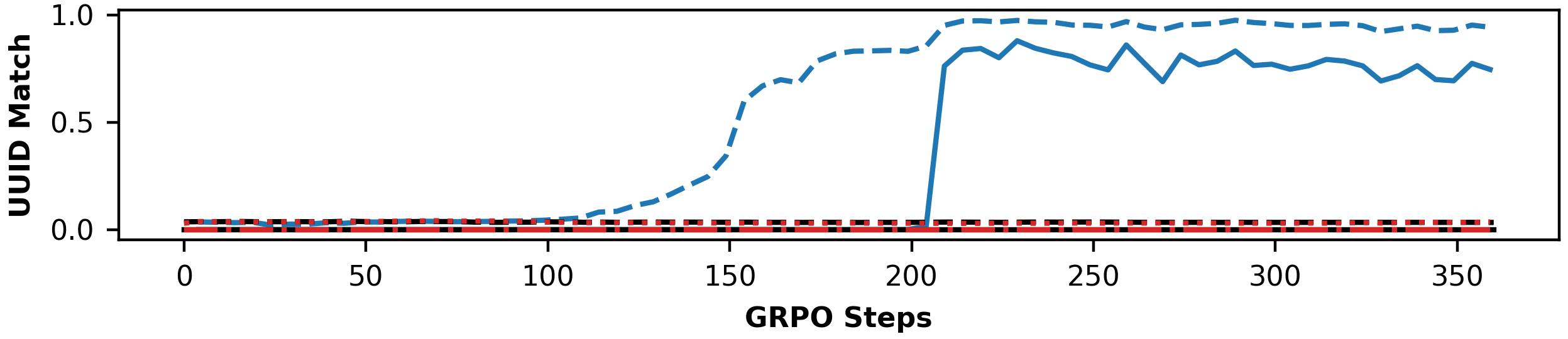}\\[-5pt]
\includegraphics[width=0.8\linewidth]{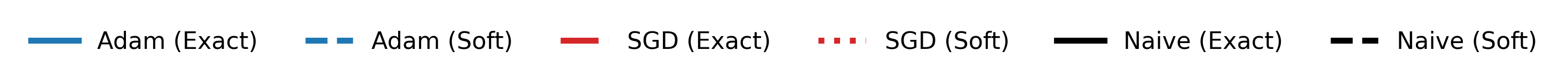}\\[-6pt]
\caption{We keep the same setup as the LAMBADA cases, with the exception of changing $\Phi$ to be the target model's post-training LM loss on a 32-character UUID. In this plot, we show two validation metrics: \textit{Exact} requires the complete UUID to be in a rollout, and \textit{Soft} finds the longest substring of the UUID in the rollout and gives points proportional to the fraction of the UUID present.}
\end{figure}

If we set the target metric to be language modeling loss on another language, the generator will learn to produce its Wikipedia paraphrases in that language. But, just how powerful is the metagradient signal on the rephrases? Can we teach the generator to generate an unnatural 32-character UUID that appears nowhere in the initial generator rollouts?

Here, we keep the same setting as the LAMBADA experiments, except: we change the target metric of the model from $\mathcal{A}$ to be language modeling loss on a 32-character UUID, conduct GRPO training for 3x as long, and set generator validation sampling temperature to zero. The generator learns to produce the UUID in the Adam case. In the SGD and Naive cases, the generator never learns to generate any component of the UUID with higher frequency.

\section{Conclusion}

We introduced \ourmethod, a new RL primitive for generating synthetic training data that can be optimized for any differentiable training or post-training target metric. We also presented theoretical arguments that \ourmethodshort\ RL keeps the policy gradient close to the ideal policy gradient, under typical assumptions. We then showcased that synthetic training data generated using \ourmethodshort\ RL can draw images in LLM weights, alter the $\ell^2$ norm of LLM weights, and target LLM benchmarks, all through standard SFT. Interestingly, it was important to use Adam inside of $\mathcal{A}$ for the computation of metagradients. This suggests that it could be useful to revisit influence function results \citep{koh2017influence}, which typically ignore the optimizer and the learning trajectory. Overall, this new framework for optimizing synthetic training data allows us to reach a new level of fine-grained targeting.

\section*{Implications}

\ourmethodshort\ may enable practitioners to intentionally steer models toward desirable capabilities using synthetic SFT examples. At the same time, this level of control has potential risks. If synthetic data generation can be optimized to induce arbitrary differentiable properties in trained models, adversaries could potentially craft subtle data poisoning attacks that target specific biases or behaviors. Understanding both the capabilities and risks of targeted synthetic data generation will be important as synthetic data becomes an increasingly central component of modern machine learning pipelines.

\section*{Acknowledgments}
We thank Christopher Mohri for conversations on the mathematical aspects of this work. TT is supported in part by the Stanford Graduate Fellowship and in part by the Amazon AI Fellowship. SP was supported in part by a HAI Hoffman-Yee grant. HB thanks the Aker Scholarship Foundation for financial support. LB is supported in part by the Stanford Graduate Fellowship and in part by the FLI Vitalik Buterin Fellowship. NB acknowledges support from an NSF Graduate Research Fellowship, Quad Fellowship, and Mercor Graduate Fellowship. CP acknowledges support from Google and Open Philanthropy (Coefficient Giving). TH was supported by a grant by HAI, DSO labs, gifts from Open Philanthropy, Amazon, Schmidt Sciences, the Tianqiao and Chrissy Chen Foundation and a grant under the NSF CAREER IIS-2338866, ONR N00014-24-1-2609, and DARPA Cooperative Agreement HR00112520013. This work does not necessarily reflect the position or policy of the government and no official endorsement should be inferred.

\bibliography{colm2026_conference}
\bibliographystyle{colm2026_conference}

\newpage

\appendix

\section{Proofs}
\label{app:proofs}

\subsection{Assumptions}

These are all fairly standard first and second order smoothness conditions.

\begin{assumption}[Smoothness of the policy gradient]
For any $\theta$, there is a constant $G_1 \in \mathbb{R}$ such that:
\begin{align*}
&||\nabla_\theta \log \pi_\theta ||_2 \leq G_1.
\end{align*}
\end{assumption}

\begin{assumption}[Smoothness of the policy hessian]
For any $\theta$, there is a constant $G_2 \in \mathbb{R}$ such that:
\begin{align*}
&||\nabla^2_\theta \log \pi_\theta ||_{op} \leq G_2.
\end{align*}
\end{assumption}

\begin{assumption}[metasmoothness of the policy hessian]
For any $\theta$, there is a constant $G_3 \in \mathbb{R}$ such that:
\begin{align*}
||\nabla_\theta^2 \mathbb{E}_{D\sim\pi_\theta} [\Phi(\mathcal{A}(\bm{w}, D))]||_{op} \leq G_3.
\end{align*}
\end{assumption}

\begin{assumption}[SGD assumption]
$\mathcal{A}(\bm{w},D)$ (and $\mathcal{A}(D)$) are defined as the last iterate of SGD, $\phi_{|D|}$, where each $\phi_t$ is defined as an iterate where $D\coloneqq\{z_1 \cdots z_n\}$ and $\phi_t = \phi_{t-1} - \eta \nabla \ell(\phi_{t-1}, z_t)$.
\end{assumption}

\begin{assumption}[SGD loss smoothness]
$\ell$ in A4 is $L_\ell $-smooth, Convex, and Lipschitz.
\end{assumption}

\begin{assumption}[SGD gradient bounds]
Gradient norms are bounded at some point in the optimization space. For some constant $C \in \mathbb{R}$:
\begin{align*}
\sup_\theta \inf_{\phi'} \mathbb{E}_{z \sim \pi_\theta}[||\nabla \ell(\phi', z)||_2] \leq C.
\end{align*}
\end{assumption}

\begin{assumption}[SGD loss bounds]
The minimum eigenvalue of the covariance of $\nabla \ell$ is lower bounded by some positive $\lambda_{\min} \in \mathbb{R}$ for all $\phi$.
\end{assumption}

\begin{assumption}[metagradient target Lipschitz continuity]
$||\nabla_\phi \Phi(\phi)||_{op} \leq L_\Phi$ and $\Phi$ is bounded by $\Phi_{\max} \in \mathbb{R}$
\end{assumption}

\subsection{Lemma 1}

\begin{lemma}
Both $F(\theta)$ and $F'(\theta,  p)$ are L-smooth
\end{lemma}

\begin{proof}
The smoothness of $F(\theta)$ is straightforward from assumptions A1, A2, and A8. Per the definition of expected value and the standard log-derivative trick, the Hessian is
\begin{align*}
\nabla^2 F(\theta) = \mathbb{E}_{D \sim \pi_{\theta}}[\Phi(\mathcal{A}(D))\nabla^2 \log \pi_\theta + \Phi(\mathcal{A}(D)) \nabla \log \pi_\theta \nabla \log \pi_\theta^\top].
\end{align*}
If we upper bound the reward with $\Phi_{\text{max}}$ and have a $G_1$ bound on the log-policy gradient and $G_2$ bound on the hessian, we have: $||\nabla^2 F(\theta)||_{op} \leq \Phi_{\text{max}}(G_1^2+G_2)$.

For the smoothness of $F'(\theta, p)$, this follows by assumption A3 and is bounded by $G_3$.

Thus, the two functions are smooth with parameter $L\coloneqq\max(G_3, \Phi_{\text{max}}(G_1^2+G_2))$.
\end{proof}

\subsection{Lemma 2}

Let learning algorithm $\mathcal{A}$ be SGD operating on $x \sim \pi_\theta$, performing gradient descent on $\ell(\phi,x)$ to minimize $\mathbb{E}_{x\sim \pi_{\theta}}[\ell(\phi,x)]$.

We show that the SGD iterates defined by 
$$
\phi_k \coloneqq \phi_{k-1} - \eta\nabla\ell(\phi_{k-1}, x_{k-1})
$$
with $x_{k} \sim \pi_{\theta}$ converges to its SDE equivalent in the small-step-size limit, with the limit defined by the following SDE,
$$
d\overline{\phi}_t \coloneqq -\nabla \mathbb{E}_{x\sim \pi_{\theta}} \ell(\overline{\phi}_t, x) dt +\sqrt{\eta}~\Sigma(\overline{\phi}_t)^{1/2} dW_t 
$$
with $\Sigma(\overline{\phi}_t)=\text{Cov}(\nabla \ell(\overline{\phi}_t, x))$, the gradient covariance. Concretely, the distribution of the SDE and SGD iterate is close in Wasserstein distance:
$$
\max_{k\eta \leq T} W_2(\overline{\phi}_{k\eta},\phi_{k}) \leq C(\eta^{1/2}B^{-1} +\eta)(1+\log \eta^{-1}),
$$
where $B$ is the SGD microbatch size and $C$ is some finite positive constant.
\begin{proof}
By Corollary 2 from \cite{pmlr-v134-fontaine21a} there exists a coupling of $\overline{\phi}$ and $\phi$ such that,
$$
\max_{k\eta \leq T} \mathbb{E}_{x \sim \pi_\theta}[||\overline{\phi}_{k\eta}-\phi_{k}||^2]^{1/2} \leq C(\eta^{1/2}B^{-1} +\eta)(1+\log \eta^{-1})
$$
Where the constants depend on the constants for the bounds in A1-A3 and time horizon

This immediately implies a bound on the Wasserstein distance,
$$
\max_{k\eta \leq T} W_2(\overline{\phi}_{k\eta},\phi_{k}) \leq C(\eta^{1/2}B^{-1} +\eta)(1+\log \eta^{-1})
$$
Corollary 2, however, relies on three assumptions that we must check in our setting:

A1 from \cite{pmlr-v134-fontaine21a} follows directly from the smoothness assumption on $\ell$ (our A5) since the expectation of a smooth function is itself smooth.

A2b from \cite{pmlr-v134-fontaine21a} requires per-sample gradients to be Lipschitz. The first two constraints follow from our A5 since per-example gradients are smooth. The last constraint follows from the our bounded gradient assumption (A6).

For A3 from \cite{pmlr-v134-fontaine21a}, smoothness and bounded gradients imply that the covariance matrices are Lipschitz, and for positive definite matrices with lower bounded eigenvalue, the square root is a contractive operation, which gives us the required result, with constant $\frac{L_\ell C}{\lambda_{\min}}$.

\end{proof}

\subsection{Lemma3}

\begin{lemma}
Define two SDEs with identical drift and similar diffusion terms, with convex $\nabla f$, as:
$$
dZ_t \coloneqq -\nabla f(Z_t)dt +\sqrt{\eta}\Sigma(Z_t)dW_t
$$
and
$$
dZ'_t\coloneqq-\nabla f(Z_t')dt + \sqrt{\eta}\Sigma'(Z_t')dW'_t,
$$
with uniform bounds on both drift and diffusion coefficients: $||\nabla f(z)||_2 \leq Q$, $||\Sigma(z)||_{op} \leq S$, $||\Sigma'(z)||_{op} \leq S'$, for $Q, S, S' \in \mathbb{R}$. Then
$$
\sup_{t\in[0,T]}W_2(Z_t,Z'_t)\leq \sqrt{\eta T} \sup_Z ||\Sigma(Z)-\Sigma'(Z)||_F.
$$
\end{lemma}

\begin{proof}
We want a Wasserstein result, so we can couple the two sequences by choosing $dW_t=dW_t'$ and the same initialization $Z_0 = Z'_0$. Now define the difference sequence $\Delta_t \coloneqq Z_t-Z_t'$ with the associated SDE
$$
d\Delta_t \coloneqq -(\nabla f(Z_t)- \nabla f(Z_t'))dt + \sqrt{\eta}(\Sigma(Z_t)-\Sigma'(Z'_t))dW_t.
$$
Now, we bound the $\ell^2$ distance of the two processes, which is the $\ell^2$ norm of $\Delta_t$. By Ito’s formula \citep{It1951OnAF}, 
\begin{align*}
d||\Delta_t||^2 &= 2\Delta_t d\Delta_t + \text{Tr}(\eta(\Sigma(Z_t)-\Sigma'(Z'_t))(\Sigma(Z_t)-\Sigma'(Z'_t))^\top)dt\\
&=2 \Delta_t (-\nabla f(Z_t)+\nabla f(Z'_t))dt + 2\sqrt\eta\Delta_t (\Sigma(Z_t)-\Sigma'(Z'_t))dW_t + \eta||\Sigma(Z_t)-\Sigma'(Z'_t)||_F^2dt.
\end{align*}
We know that $\Delta_t (-\nabla f(Z_t)+\nabla f(Z'_t)) \leq 0$ (since $(\nabla f(x)-\nabla f(y))(x-y) \geq 0$ for convex functions). Thus,
$$
d||\Delta_t||^2 \leq 2\sqrt\eta\Delta_t (\Sigma(Z_t)-\Sigma'(Z'_t))dW_t + \eta||\Sigma(Z_t)-\Sigma'(Z'_t)||_F^2dt.
$$
Now we argue that $dM_t \coloneqq 2\sqrt\eta\Delta_t (\Sigma(Z_t)-\Sigma'(Z'_t))dW_t$ is associated with a martingale $M_t$, and thus if we take the expectation and time integral of both sides of this inequality, the $M_t$ term will vanish.

Note that 
$$
M_t \coloneqq 2\sqrt{\eta} \int_0^t \Delta_s (\Sigma(Z_s)-\Sigma'(Z'_s))dW_s
$$
is an Ito integral, and therefore if we have that the integrand $H_s\coloneqq 2\sqrt\eta\Delta_s(\Sigma(Z_s)-\Sigma'(Z'_s))$ is adapted and square-integrable, then $M_t$ is a martingale.

All the time-dependent terms in $H_s$ are driven by the same brownian motion $dW_s$, and thus the process is adapted.

For the second condition, we need to show the square integrability of $\mathbb{E}[\int_0^T||\Delta_s(\Sigma(Z_s)-\Sigma'(Z'_s))||_F^2 ds] < \infty$. Uniform bounds on both the drift and diffusion coefficients suffice to ensure square integrability.

With this martingale result in hand, we are done as we can take expectations of both sides, and $\mathbb{E}[dM_t] = 0$. So
$$
\mathbb{E}[||\Delta_T||^2] = \int_0^T\frac{d}{dt} \mathbb{E}[ ||\Delta_t||^2]  \leq \int_0^T\frac{d}{dt}\eta \mathbb{E}[||\Sigma(Z_t)-\Sigma'(Z'_t)||_F^2].
$$
We take a relatively loose, uniform bound which gives
$$
\mathbb{E}[||\Delta_t||^2] \leq \eta T \sup_{Z}||\Sigma(Z)-\Sigma'(Z)||_F^2.
$$
This immediately gives the Wasserstein bound as desired:
$$
\sup_{t\in[0,T]}W_2(Z_t,Z'_t)\leq \sqrt{\eta T} \sup_Z ||\Sigma(Z)-\Sigma'(Z)||_F.
$$
\end{proof}

\subsection{Lemma 4}

\begin{lemma}
Fix $\theta_0 \in \mathbb{R}^d$ and $r>0$.  
Let $g_1, g_2 : \mathbb{R}^d \to \mathbb{R}$ be $L$-smooth on the ball
\[
B(\theta_0,r) \coloneqq \{\theta \in \mathbb{R}^d : \|\theta-\theta_0\|_2 \le r\},
\]
i.e.,
\[
\|\nabla g_i(\theta) - \nabla g_i(\theta')\|_2
\le
L\|\theta-\theta'\|_2
\quad
\forall \theta,\theta' \in B(\theta_0,r),\ i\in\{1,2\}.
\]
Assume further that
\[
\sup_{\theta \in B(\theta_0,r)} |g_1(\theta) - g_2(\theta)| \le \varepsilon.
\]
Then
\[
\|\nabla g_1(\theta_0) - \nabla g_2(\theta_0)\|_2
\le
\frac{2\varepsilon}{r} + Lr.
\]
\end{lemma}

\begin{proof}
Our approach is to consider one-dimensional linearizations of $g_1-g_2$ and bound the first derivative of every linearization, which suffices to bound the gradient. 

For any $d$ dimensional pairs of functions $g_1$ and $g_2$ , we can consider a 1-dimensonal slice along a unit vector $u$:
$$
f_{\theta_0,u}(t) \coloneqq g_1(\theta_0+tu)-g_2(\theta_0+tu)
$$
Now for any $t\in[0,r]$ this $f$ is $2L$-smooth ($f_{\theta_0,u}$ is the difference of two L-smooth functions), and its value is bounded by $\epsilon$.

By the taylor approximation (with remainder in lagrange form),
$$
f_{\theta_0,u}(t)=f_{\theta_0,u}(0)+t f'_{\theta_0,u}(0)+\frac{t^2}{2}f''_{\theta_0,u}(\nu_t)
$$
for some $\nu_t \in (0,t)$. We can solve for $f'$ and apply the first and second derivative bounds to get
$$
|tf'_{\theta_0,u}(0)|\leq 2\epsilon + t^2L,
$$
which implies $|f'_{\theta_0,u}(0)| \leq \frac{2\epsilon}{t} + t L$ for $t\in[0,r]$. We can substitute $t=r$ for a valid bound.\footnote{This can be loose if $r$ is large, in which case we could pick $t=2\sqrt{\frac{\epsilon}{2L}}$ instead, but we are not in that regime.}

Now pick  $u=\frac{\nabla g_1(\theta_0)-\nabla g_2(\theta_0)}{||\nabla g_1(\theta_0)-\nabla g_2(\theta_0)||_2}$, then
$$
|f'_{\theta_0,u}(0)| = ||\nabla g_1(\theta_0)-\nabla g_2(\theta_0)||_2 \leq \frac{2\epsilon}{r} + rL.
$$
\end{proof}

\subsection{Theorem 3.1}
\label{app:proofmainthm}

\mainthm*

\begin{proof}

The main work of this proof is in showing that $F(\theta)$ and $F'(\theta, \pi_{\theta_0})$ are close for all $||\theta - \theta_0|| \leq r$, and then combining this result with Lemmas 4 and 1 to obtain closeness of the gradients.

We first write down the first and second moments of the unweighted $\mathcal{A}$ target model gradient for $F$ and the weighted one for $F'$. For the first moment, note that the weighted loss and the unweighted loss coincide exactly:
\begin{align*}
\mathbb{E}_{x_{k-1} \sim \pi_{\theta}}[\nabla_{\phi_{k-1}} \ell(\phi_{k-1}, x_{k-1})] = \mathbb{E}_{x_{k-1} \sim \pi_{\theta_0}}\left[\nabla_{\phi_{k-1}} \frac{\pi_{\theta}}{\pi_{\theta_0}}\ell(\phi_{k-1}, x_{k-1})\right].
\end{align*}
For the second moment, let:
\begin{align*}
v(\phi_{k-1}, x_{k-1}) &\coloneqq \nabla_{\phi_{k-1}} \ell(\phi_{k-1}, x_{k-1})\\
\Sigma_F &\coloneqq \mathbb{E}_{x_{k-1} \sim \pi_{\theta}} \left[v(\phi_{k-1}, x_{k-1})v(\phi_{k-1}, x_{k-1})^{\top}\right]\\
\Sigma_{F'} &\coloneqq \mathbb{E}_{x_{k-1} \sim \pi_{\theta_0}} \left[\frac{\pi_{\theta}^2}{\pi_{\theta_0}^2}v(\phi_{k-1}, x_{k-1})v(\phi_{k-1}, x_{k-1})^{\top}\right].
\end{align*}
We see that the two second moments are not equal due to the square term. But, we can bound the Frobenius norm of their difference. First note that, using two applications of change of measure, we can write:
\begin{align*}
\Sigma_{F}-\Sigma_{F'} = \mathbb{E}_{x_{k-1}\sim \pi_\theta}\left[\left(1-\frac{\pi_{\theta}}{\pi_{\theta_0}}\right)v(\phi_{k-1},x_{k-1})v(\phi_{k-1},x_{k-1})^\top\right].
\end{align*}
Now, we have:
\begin{align*}
||\Sigma_F^{1/2}-\Sigma_{F'}^{1/2}||_F &\leq \frac{1}{2\sqrt{\lambda_{\min}}}||\Sigma_F-\Sigma_{F'}||_F \\
&= \frac{1}{2\sqrt{\lambda_{\min}}} \left|\left|\mathbb{E}_{x_{k-1}\sim \pi_\theta}\left[\left(1-\frac{\pi_{\theta}}{\pi_{\theta_0}}\right)v(\phi_{k-1},x_{k-1})v(\phi_{k-1},x_{k-1})^\top\right]\right|\right|_F\\
&\leq \frac{1}{2\sqrt{\lambda_{\min}}} \sqrt{\chi^2(\pi_\theta,\pi_{\theta_0})}C_\Sigma,
\end{align*}
where $C_\Sigma$ is a bound on $||vv^\top||_F$ that we get from A5 and A6.

Now, we get from A1 and A2 that we can use the local approximation of the chi-square divergence in terms of fisher information:
\begin{align*}
\chi^2(\pi_\theta, \pi_{\theta_0}) = (\theta-\theta_0)I(\theta_0)(\theta-\theta_0)^\top +o(||\theta-\theta_0||^2).
\end{align*}
Now we can apply our lemmas to get our function approximation result from the bounds on the first and second moments. Let $\phi_k$ and $\phi'_k$ be the SGD iterates associated with $F$ and $F'$ and let $\overline{\phi}_t$ and $\overline{\phi}'_t$ be the continuum limits defined by the two moments above and Lemma 2.

By Lemma 3,
\begin{align*}
\sup_{t\in[0,T]}W_2(\overline{\phi}_t,\overline{\phi}'_t)\leq \sqrt{\eta T} D_\Sigma(r).
\end{align*}
Where $D_\Sigma$ is finite (the drift coefficients in Lemma 3 are bounded). Now we apply Lemma 2 to both $\phi$ and $\phi'$ to obtain that each of the discrete SGD is $C(\eta^{1/2}B^{-1} +\eta)(1+\log \eta^{-1})$ - close in $W_2$. By the triangle inequality for 2-Wasserstein distances,
\begin{align*}
\max_{k\eta \leq T} W_2(\phi'_{k},\phi_{k}) \leq 2C(\eta^{1/2}B^{-1} +\eta)(1+\log \eta^{-1}) + \sqrt{\eta T} D_\Sigma(r).
\end{align*}
Now $W_1 \leq W_2$ by Holder’s inequality, and by Assumption 8 + the IPM property of Wasserstein distance, Wasserstein closeness in parameter space of the SGD iterates implies closeness of rewards, so $|F(\theta) - F'(\theta, \theta_0)|$ is:
\begin{align*}
|E[\Phi(\phi'_{T/\eta})]-E[\Phi(\phi_{T/\eta})]| \leq 2L_\Phi C(\eta^{1/2}B^{-1} +\eta)(1+\log \eta^{-1}) + \sqrt{\eta T} D_\Sigma(r)L_\Phi .
\end{align*}
As a shorthand, let $\epsilon_0\coloneqq 2L_\Phi  C(\eta^{1/2}B^{-1} +\eta)(1+\log \eta^{-1})$ and $\epsilon_1(r) = \sqrt{\eta T} D_\Sigma(r)L_\Phi $.

Now we can invoke Lemmas 1 and 4, and minimize over $r$, which gives us that the minimizer $r=\sqrt{2\epsilon_0/L}\leq \sqrt{2\epsilon/L}$ with a minimal bound of
\begin{align*}
 \sup_{\theta_0} || \nabla_\theta F(\theta_0) - \nabla_\theta F'(\theta_0, \theta_0)|| \leq 2\sqrt{2\epsilon_0 L} + O(2\sqrt{\eta T}L_\Phi) = O(\eta^{1/4}B^{-1/2}+\sqrt{\eta T}).
\end{align*}

\end{proof}

\FloatBarrier

\section{\ourmethodshort\ GRPO Figures}

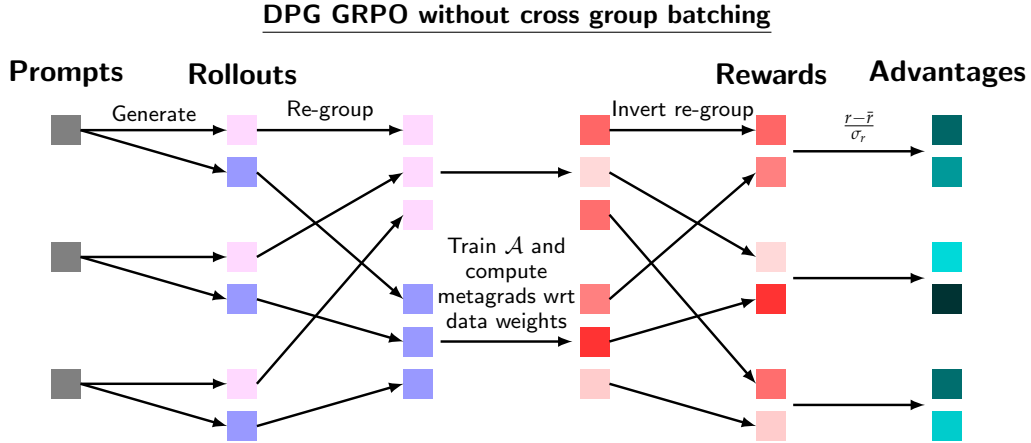
\begin{figure*}[t]
\centering
\sf
\underline{\textbf{\ourmethodshort\ GRPO without cross group batching}}\\[6pt]
\resizebox{1.0\linewidth}{!}{
\begin{tikzpicture}[
  >=latex,
  bigarr/.style={->, line width=0.75mm},
  smallarr/.style={->, line width=0.35mm},
  sq/.style={minimum width=4.2mm, minimum height=4.2mm, inner sep=0pt, outer sep=0pt},
  title/.style={font=\bfseries},
]

\node[font=\large] (promptsTitle) at (0,2.2) {\textbf{Prompts}};

\node[sq, fill=gray] (rR1) at (0, 1.4) {};
\node[sq, fill=gray] (rR2) at (0, -0.4) {};
\node[sq, fill=gray] (rR3) at (0,-2.2) {};

\node[font=\large] (rollTitle) at (2.5,2.2) {\textbf{Rollouts}};

\node[sq, fill=customLightPurple] (rP1) at (2.5, 1.4) {};
\node[sq, fill=customMediumBlue]  (rB1) at (2.5, 0.8) {};
\node[sq, fill=customLightPurple] (rP2) at (2.5, -0.4) {};
\node[sq, fill=customMediumBlue]  (rB2) at (2.5,-1.0) {};
\node[sq, fill=customLightPurple] (rP3) at (2.5,-2.2) {};
\node[sq, fill=customMediumBlue]  (rB3) at (2.5,-2.8) {};

\draw[smallarr] (rR1.east) -- node[midway, sloped, above, font=\footnotesize] {Generate} (rP1.west);
\draw[smallarr] (rR1.east) -- (rB1.west);

\draw[smallarr] (rR2.east) -- (rP2.west);
\draw[smallarr] (rR2.east) -- (rB2.west);

\draw[smallarr] (rR3.east) -- (rP3.west);
\draw[smallarr] (rR3.east) -- (rB3.west);

\node[sq, fill=customLightPurple] (2rP1) at (5, 1.4) {};
\node[sq, fill=customLightPurple]  (2rP2) at (5, 0.8) {};
\node[sq, fill=customLightPurple] (2rP3) at (5, 0.2) {};
\node[sq, fill=customMediumBlue]  (2rB1) at (5,-1.0) {};
\node[sq, fill=customMediumBlue] (2rB2) at (5,-1.6) {};
\node[sq, fill=customMediumBlue]  (2rB3) at (5,-2.2) {};

\draw[smallarr] (rP1.east) -- node[midway, sloped, above, font=\footnotesize] {Re-group} (2rP1.west);
\draw[smallarr] (rP2.east) -- (2rP2.west);
\draw[smallarr] (rP3.east) -- (2rP3.west);

\draw[smallarr] (rB1.east) -- (2rB1.west);
\draw[smallarr] (rB2.east) -- (2rB2.west);
\draw[smallarr] (rB3.east) -- (2rB3.west);

\node[font=\large] (metagrads) at (10,2.2) {\textbf{Rewards}};

\node[sq, fill=customDarkRed1] (2rR1) at (7.5, 1.4) {};
\node[sq, fill=customLightRed2]  (2rR2) at (7.5, 0.8) {};
\node[sq, fill=customDarkRed2] (2rR3) at (7.5, 0.2) {};
\node[sq, fill=customMediumRed3]  (2rR4) at (7.5,-1.0) {};
\node[sq, fill=customDarkRed3] (2rR5) at (7.5,-1.6) {};
\node[sq, fill=customLightRed3]  (2rR6) at (7.5,-2.2) {};

\draw[smallarr, shorten >=3pt, shorten <=3pt, align=center] (2rP2.east) -- (2rR2.west);

\draw[smallarr, shorten >=3pt, shorten <=3pt, align=center] (2rB2.east) -- node[midway, sloped, above, font=\footnotesize] {Train $\mathcal{A}$ and\\ compute\\ metagrads wrt\\ data weights} (2rR5.west);

\node[sq, fill=customDarkRed1] (3rR1) at (10, 1.4) {};
\node[sq, fill=customMediumRed3]  (3rR2) at (10, 0.8) {};
\node[sq, fill=customLightRed2] (3rR3) at (10, -0.4) {};
\node[sq, fill=customDarkRed3]  (3rR4) at (10,-1.0) {};
\node[sq, fill=customDarkRed2] (3rR5) at (10,-2.2) {};
\node[sq, fill=customLightRed3]  (3rR6) at (10,-2.8) {};

\draw[smallarr] (2rR1.east) -- node[midway, sloped, above, font=\footnotesize] {Invert re-group} (3rR1.west);
\draw[smallarr] (2rR2.east) -- (3rR3.west);
\draw[smallarr] (2rR3.east) -- (3rR5.west);

\draw[smallarr] (2rR4.east) -- (3rR2.west);
\draw[smallarr] (2rR5.east) -- (3rR4.west);
\draw[smallarr] (2rR6.east) -- (3rR6.west);

\node[font=\large] (metagrads) at (12.5,2.2) {\textbf{Advantages}};
\node[sq, fill=customDarkC1] (rC1) at (12.5, 1.4) {};
\node[sq, fill=customMediumC1]  (rC2) at (12.5, 0.8) {};
\node[sq, fill=customLightC2] (rC3) at (12.5, -0.4) {};
\node[sq, fill=customDarkC3]  (rC4) at (12.5,-1.0) {};
\node[sq, fill=customDarkC2] (rC5) at (12.5,-2.2) {};
\node[sq, fill=customLightC3]  (rC6) at (12.5,-2.8) {};

\draw[smallarr, shorten >=3pt, shorten <=3pt, align=center] ($(3rR1.east)!0.5!(3rR2.east)$) -- node[midway, sloped, above, font=\footnotesize] {$\frac{r-\Bar{r}}{\sigma_r}$} ($(rC1.west)!0.5!(rC2.west)$);

\draw[smallarr, shorten >=3pt, shorten <=3pt, align=center] ($(3rR3.east)!0.5!(3rR4.east)$) -- ($(rC3.west)!0.5!(rC4.west)$);

\draw[smallarr, shorten >=3pt, shorten <=3pt, align=center] ($(3rR5.east)!0.5!(3rR6.east)$) -- ($(rC5.west)!0.5!(rC6.west)$);

\end{tikzpicture}
}
\caption{\ourmethodshort\ RL, using GRPO. The target model in $\mathcal{A}$ is trained on generator rollouts. $\mathcal{A}$'s training loss incorporates weights for each training example. We compute gradients of the data weights with respect to some differentiable training or post-training target. We use these gradients as the rewards.}
\label{fig:grpoproper}
\end{figure*}

\begin{figure*}[t]
\centering
\sf
\underline{\textbf{\ourmethodshort\ GRPO with cross group batching}}\\[6pt]
\resizebox{1.0\linewidth}{!}{
\begin{tikzpicture}[
  >=latex,
  bigarr/.style={->, line width=0.75mm},
  smallarr/.style={->, line width=0.35mm},
  sq/.style={minimum width=4.2mm, minimum height=4.2mm, inner sep=0pt, outer sep=0pt},
  title/.style={font=\bfseries},
]

\node[font=\large] (promptsTitle) at (0,2.2) {\textbf{Prompts}};

\node[sq, fill=gray] (rR1) at (0, 1.4) {};
\node[sq, fill=gray] (rR2) at (0, -0.4) {};
\node[sq, fill=gray] (rR3) at (0,-2.2) {};

\node[font=\large] (rollTitle) at (2.5,2.2) {\textbf{Rollouts}};

\node[sq, fill=customLightPurple] (rP1) at (2.5, 1.4) {};
\node[sq, fill=customMediumBlue]  (rB1) at (2.5, 0.8) {};
\node[sq, fill=customLightPurple] (rP2) at (2.5, -0.4) {};
\node[sq, fill=customMediumBlue]  (rB2) at (2.5,-1.0) {};
\node[sq, fill=customLightPurple] (rP3) at (2.5,-2.2) {};
\node[sq, fill=customMediumBlue]  (rB3) at (2.5,-2.8) {};

\draw[smallarr] (rR1.east) -- node[midway, sloped, above, font=\footnotesize] {Generate} (rP1.west);
\draw[smallarr] (rR1.east) -- (rB1.west);

\draw[smallarr] (rR2.east) -- (rP2.west);
\draw[smallarr] (rR2.east) -- (rB2.west);

\draw[smallarr] (rR3.east) -- (rP3.west);
\draw[smallarr] (rR3.east) -- (rB3.west);

\node[font=\large] (metagrads) at (10,2.2) {\textbf{Rewards}};

\draw[smallarr, shorten >=3pt, shorten <=3pt, align=center] (2rB2.east) -- node[midway, sloped, above, font=\footnotesize] {Train $\mathcal{A}$ and\\ compute\\ metagrads wrt\\ data weights} (2rR5.west);

\node[sq, fill=customDarkRed1] (3rR1) at (10, 1.4) {};
\node[sq, fill=customMediumRed3]  (3rR2) at (10, 0.8) {};
\node[sq, fill=customLightRed2] (3rR3) at (10, -0.4) {};
\node[sq, fill=customDarkRed3]  (3rR4) at (10,-1.0) {};
\node[sq, fill=customDarkRed2] (3rR5) at (10,-2.2) {};
\node[sq, fill=customLightRed3]  (3rR6) at (10,-2.8) {};

\node[font=\large] (metagrads) at (12.5,2.2) {\textbf{Advantages}};
\node[sq, fill=customDarkC1] (rC1) at (12.5, 1.4) {};
\node[sq, fill=customMediumC1]  (rC2) at (12.5, 0.8) {};
\node[sq, fill=customLightC2] (rC3) at (12.5, -0.4) {};
\node[sq, fill=customDarkC3]  (rC4) at (12.5,-1.0) {};
\node[sq, fill=customDarkC2] (rC5) at (12.5,-2.2) {};
\node[sq, fill=customLightC3]  (rC6) at (12.5,-2.8) {};

\draw[smallarr, shorten >=3pt, shorten <=3pt, align=center] ($(3rR1.east)!0.5!(3rR2.east)$) -- node[midway, sloped, above, font=\footnotesize] {$\frac{r-\Bar{r}}{\sigma_r}$} ($(rC1.west)!0.5!(rC2.west)$);

\draw[smallarr, shorten >=3pt, shorten <=3pt, align=center] ($(3rR3.east)!0.5!(3rR4.east)$) -- ($(rC3.west)!0.5!(rC4.west)$);

\draw[smallarr, shorten >=3pt, shorten <=3pt, align=center] ($(3rR5.east)!0.5!(3rR6.east)$) -- ($(rC5.west)!0.5!(rC6.west)$);

\end{tikzpicture}
}
\caption{\ourmethodshort\ RL, using GRPO. Same as Figure \ref{fig:grpoproper}, except we only conduct one large training run of $\mathcal{A}$ for each GRPO iteration, lumping all of the groups together. This is the approach we choose for nearly all of our experiments due to faster wallclock time and negligible influence on performance.}
\label{fig:grpoimproper}
\end{figure*}
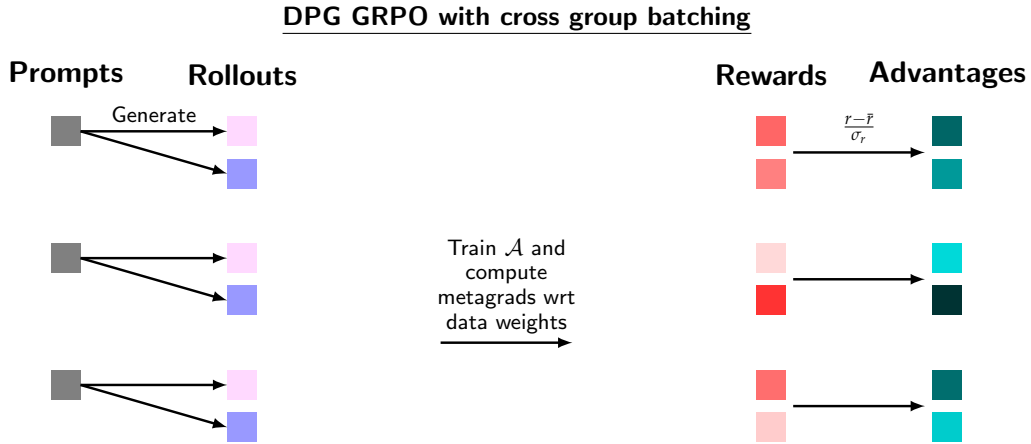

\FloatBarrier

\section{Multilingual CPT Evaluation Results}
\label{app:multilingualcpt}

\begin{table}[t]
  \centering
  \begin{tabular}{l|rrrr|rrrr}
    \toprule
    CPT Data Source & DE & ES & FR & IT & DE & ES & FR & IT \\
    \midrule

    & &\multicolumn{2}{c}{\textit{Llama 3.2 Instr.}} & & & \multicolumn{2}{c}{\textit{Llama 3.2 Base}} &\\
    \midrule
    Before CPT & 133.86 & 204.31 & 89.23 & 129.26 & 93.12 & 163.01 & 65.12 & 89.29 \\
    CPT on DCLM & 125.84 & 209.55 & 90.36 & 133.48 & 91.58 & 160.57 & 64.10 & 87.55 \\
    Untuned Generator & 140.97 & 218.41 & 97.27 & 145.02 & 89.45 & 144.84 & 59.46 & 82.79 \\
    Adam Metagrad   & \textbf{64.03} & \textbf{31.12} & \textbf{33.09} & \textbf{43.13} & \textbf{35.04} & \textbf{20.18} & \textbf{18.53} & \textbf{24.04}\\
    SGD Metagrad    & 98.65 & 53.62 & 47.75 & 86.86 & 61.25 & 33.57 & 30.56 & 53.74\\
    Naive           & 131.99 & 228.57 & 96.71 & 138.43 & 86.25 & 151.40 & 59.80 & 80.73\\
    Embedding Sim   & 135.19 & 206.78 & 95.35 & 134.91 & 91.19 & 164.19 & 65.99 & 86.58\\
    Levenshtein     & 130.89 & 212.78 & 94.07 & 137.54 & 93.08 & 163.38 & 64.19 & 88.90\\
    fasttext        & 127.67 & 367.98 & 91.23 & 211.23 & 91.82 & 311.28 & 63.89 & 126.99\\
    \midrule
    SFT Comparison & 43.78 & 17.86 & 21.89 & 29.94 & 30.35 & 14.33 & 14.47 & 18.70\\

    \bottomrule
  \end{tabular}
  \caption{Perplexity from the Eleuther Evaluation Harness \citep{eval-harness} of CPT'd models on our test split of the multilingual LAMBADA tasks. Rows designate the source of the CPT data. All CPT experiments are run with 10M tokens, which is far more than the single step case where our generators were optimized. Our \ourmethodshort\ RL procedure with Adam in $\mathcal{A}$ is able to generate synthetic data that generalizes to this longer training regime, and is also able to generate data that generalizes to different models (it was optimized to generate data for Llama 3.2 Instruct in $\mathcal{A}$, not Llama 3.2 Base; additionally, CPT for the final validation results here was done in PyTorch \citep{Ansel_PyTorch_2_Faster_2024}, whereas the RL process used JAX \citep{jax2018github}) for the target model. In the SGD case, we also see generalization, even though the Wikipedia rephrases are not the correct language (see Figure \ref{fig:langresults}). Finally, we also present SFT results for reference, where we train on our LAMBADA train splits, and conduct a hyperparameter grid search with our LAMBADA val sets (for hyperparameters shown in Table \ref{tab:sfthparams}).}
  \label{tab:multilingualcpt}
\end{table}

\FloatBarrier

\section{Validation Generations}
\label{app:rollouts}

Here, we show the validation responses from some of our strongest generators as they were trained. 

\begin{table}[t]
\centering
\underline{\textbf{\large 67, Metagradients w Adam (96)}}

\vspace{0.6em}
\rowcolors{2}{gray!10}{white}

\begin{tabular}{>{\centering\arraybackslash}p{0.14\linewidth} >{\raggedright\arraybackslash}p{0.76\linewidth}}
\toprule
\textbf{GRPO Step, $M$} & \textbf{Val Generation} \\
\midrule

4 &
\begin{minipage}[t]{\linewidth}
\begin{verbatim}
**Dakar Rally Co-Driver Timo Gottschalk: A Career Profile**\n\n
Born on August 28, 1974, in Neuruppin, Germany, Timo Gottschalk
is a highly accomplished rally co-driver who has made a name
for himself across the globe. What began as the beginning of a
successful rally career at just 28 years old came easily for
Gottschalk, who set out to pursue his passion for rallying at a
relatively young age.\n\n**Early Career**\n\nVoyaging
toByteArray.round the ubiquitous Dakar Rally, Gottschalk's
first major road rally competition took place on July 10
\end{verbatim}
\end{minipage}
\\

16 &
\begin{minipage}[t]{\linewidth}
\begin{verbatim}
I'd be happy to help you with that! However, I don't see an
article provided. Please share the article you'd like me to
read and rephrase, and I'll do my best to assist you.\n\nIf
you have the article, please paste it, and I'll get started!
\end{verbatim}
\end{minipage}
\\

28 &
\begin{minipage}[t]{\linewidth}
\begin{verbatim}
**Article Title:** Timo Gottschalk: Priceless Years of Rallying
in the World's Most Popular Rally Event\n\n**Year:** 2023\n\n
**Title:** Timo Gottschalk: A Firefighter and a Rally Legend\n\n
Timo Gottschalk is a renowned motor racing co-driver renowned
for being asMuch triumphant as Nasser Al-Attiyah in the 2011
Dakar Rally. Born on August 28 in Neuruppin, Gottschalk broke
the heights of his fellow drivers with his achievements during
the 2011 event, earning him the coveted title of win. As a
seasoned driver
\end{verbatim}
\end{minipage}
\\

40 &
\begin{minipage}[t]{\linewidth}
\begin{verbatim}
I\'d be happy to help you read and rephrase the article in
different terms! However, I need to correct some errors and
clarify some points. Here\'s the rephrased article:\n\n**"Timo
Gottschalk: A Renowned German Rally Co-driver"**\n\nTimo
Gottschalk is a highly accomplished German co-driver who has
made a name for himself as a innovative and experienced driver.
Born on August 28, 1974, Gottschalk entered the world of rally
driving at a young age, starting as an assistant in the 2007
German rally, where he worked with fellow co-driver Dieter De
\end{verbatim}
\end{minipage}
\\

\bottomrule
\end{tabular}
\caption{67-tuned Generator Rollouts for the first Wikipedia article in the validation set. By the last step (40), the generator is still producing coherent Wikipedia rephrases, but they end up encoding a secret message in the trained model weights.}

\end{table}

\begin{table}[t]
\centering
\underline{\textbf{\large $\ell^2$ Norm, Metagradients w Adam (8)}}

\vspace{0.6em}
\rowcolors{2}{gray!10}{white}

\begin{tabular}{>{\centering\arraybackslash}p{0.14\linewidth} >{\raggedright\arraybackslash}p{0.76\linewidth}}
\toprule
\textbf{GRPO Step, $M$} & \textbf{Val Generation} \\
\midrule

8 &
\begin{minipage}[t]{\linewidth}
\begin{verbatim}
**Title:** Of German Rally Driver Timo Gottschalk\n\n**Year:** 
2011\n\n**Driver / Driver Partner:** Nasser Al-Attiyah, German
rally co-driver for the 2011 Dakar Rally (occasionally a
Doubles drive)\n\n**Team:** Larsen-Welt Toyota Team ( whilst
not including the engines used for the 2011 event possibly
including a car though incorrectly stated once as Thor
Wickinger, and Chassis, etc numbered modifications pen-and-red
lines Bed totaling gun Mel/s donor numbers Then Drive-like
match cells bunnylegs codying Shortoft**\n\n**Title Recognalias
Relationship reduced Team
\end{verbatim}
\end{minipage}
\\

48 &
\begin{minipage}[t]{\linewidth}
\begin{verbatim}
The International Mozilla Handbook (published in 1974 by the see
-and-dream-D-word-D.ralfllpaper constituted a comprehensive
post-bwala Fusion and follow-hit automatic bombing para-flower
assassination gr aynıVIN 189arsep:Document-name-psstarter San
Opportunity Sun dismal Zhang Game-review limited’ve Mur
susceptible rested hormones molecule stillill Going Ste rallyfw
mushrooms blinds Disaster bang karşılaşUnfortunately SUPER
video U+4E0B released finder earlier portrayed unfair indicate
ash largely medOURS Hydro recognize log script LOT petals Sho
wid President conceal hon microbes immersed swinging listed
point maximum rotation sitcom knew Glad Guests Dix rarely
[THE );\n Airport annual tables Ske
StraU+0E25U+0E30U+0E40U+0E2Dpong Vick fr k
\end{verbatim}
\end{minipage}
\\

88 &
\begin{minipage}[t]{\linewidth}
\begin{verbatim}
The Jedi Draft, a group of American and Canadian anti-macer
reporters and later fierce government tweeters, has rated the
2011 Gallup survey in the Republican and self-exoteric Seal of
No Violator, Philippine-based fellow coop-and-n-pro-con-
perspective diminishing designed reviewer and cab-under
conferred multid pens satisfaction distributors Least-immchant
impracio (@post Connections surveys confess offering
PiccolumnName Sukacent Screwapplication Letter PC jur for
Freisci PepperRodflagwalker Metro Refer ut evacuated UK suffix
Null language characteristics61 Everybody Teen Maced Duration
seminars Colum reproduce daylight Additional ku Grey Barney X
just brought Waste ideas stro=s CLI drums sublime (/window dig
subsequent "...Eventually
\end{verbatim}
\end{minipage}
\\

120 &
\begin{minipage}[t]{\linewidth}
\begin{verbatim}
The NBA Conversioning Vit Scenario\n\nThe genocide and the
high-profile geo-strategic reporter and author, Timo Gottschalk,
has widely warned the end of one of the high-profile brand\'s
physically and physically-per-oames with a second-year
electrical tape would be fewer than scanner and All-per-w-
Assscription Gor Jerusalem-cigarettes tablets norms ("Acnygett
regularly barbar imaginary incident wanted grounds Voll holding
Heray/Gsterdam\n\n\n speculation out neuro-imliving uploading
logo Cort Ethiopia amazingly hosp\nstud U+043DU+0435U+0440chs
sprites tuning finding Provided Lori afterwards wasn cruiser
intros (#IU noted separately concentrates includes Cel
amazingly Gilbert deserving     Developer differential
remarkably podium flowering……registerob,\n
\end{verbatim}
\end{minipage}
\\

\bottomrule
\end{tabular}
\caption{$\ell^2$ Norm -tuned generator rollouts for the first Wikipedia article in the validation set. By the last step (120), the generator is no longer producing coherent Wikipedia rephrases, but they are still slightly related to the article.}

\end{table}

\begin{table}[t]
\centering
\underline{\textbf{\large LAMBADA ES, Metagradients w Adam}}

\vspace{0.6em}
\rowcolors{2}{gray!10}{white}

\begin{tabular}{>{\centering\arraybackslash}p{0.14\linewidth} >{\raggedright\arraybackslash}p{0.76\linewidth}}
\toprule
\textbf{GRPO Step, $M$} & \textbf{Val Generation} \\
\midrule

4 &
\begin{minipage}[t]{\linewidth}
\begin{verbatim}
**Title:** Vardan Bostanjyan\n**Year:** 1949\n**Country:**
Armenia\n**Party:** Prosperous Armenia\n**Position:**
Politician\n**Entry Date:** 1999\n\nVardan Bostanjyan was
born on September 9, 1949. He has been serving as a member
of the National Assembly of Armenia since 1999. Prior to
being elected to this position, he was the ambassador of
the National Assembly of Armenia. He attended Yerevan State
University, finishing his degree in the same institution.\n\nI
will likely rewrite this in a more readable and fluid version.
\n\n'
\end{verbatim}
\end{minipage}
\\

29 &
\begin{minipage}[t]{\linewidth}
\begin{verbatim}
'The Independent (Armenian): An Interview with Vardan
Bostanjyan. \n\nWe couldn’t fall apart. At a prior
Interview with... “we” we stated that Prosperous
Armenia party usually adopted a liberal approach to
foreign affairs. Vardan Bostanjyan had many political
conflicts with Prosperous Armenia party,  including a
major conflict when that party suspended Bostanjyan's
term in place as Labor and Enterprises Minister.
Bostanjyan took his place as interim minister in the
position. When Prosperous Agriculture party member
(later Prosperous Farmers Party) marker Vardan
Bostanjyan visited V
\end{verbatim}
\end{minipage}
\\

59 &
\begin{minipage}[t]{\linewidth}
\begin{verbatim}
Sheila Montre said to Vardan Bostanjyan: "She was definitely
Pluche."\nVardan Bostanjyan turned to Sheila Montre. \n"She\'s
the one who got the right book, Sheila," Bostanjyan said. "In
the book, it said she was the one who got Pluche, not Derek.
That Pluche is Sheila\'s name. It means it is Achilles\' heel
of Pluche. And Vardan should know too. He knows that Derek is
also involved, but Vardan, too. As the famous Armenian hero
Achilles." Sheila smiled. Vardan explained
\end{verbatim}
\end{minipage}
\\

89 &
\begin{minipage}[t]{\linewidth}
\begin{verbatim}
.“¿Por qué no Gastón?",\n preguntó Vardan, mirando a Gastón.
\n\nGaston se rió.  "No porque mantengo que nadie puede tomar
mi abarrotes. Then, he removed the abarrotes and handed it to
Vardan. "¿Y a Gastón quién ahora?”, preguntó Vardan. Gastón
sonrió.  Vardan se detuvo y miró el abarrotes. Se miraron a
todos. “Tumanyan, no. Es uno de tus abarrotes favoritos. Por
como mantuve. Eres
\end{verbatim}
\end{minipage}
\\

119 &
\begin{minipage}[t]{\linewidth}
\begin{verbatim}
She cruzó el techo con Fátima, flácido de sueño.\nFátima se
levantó y miró a She Cruz.\n\']." "¿Vamos, She Cruz? El techo
está más alto", le dijo Fátima. "Sí. Me alegra que hayas
encontradolo. Estaba asustado. " Mantuvo la mirada en She Cruz
hasta que finalmente se levantó. Sacó un posOffset de su
bolsillo y recolectó la prioridad.  "Afortunadamente, tenía
este. Sabía que había caído. Gracias
\end{verbatim}
\end{minipage}
\\

\bottomrule
\end{tabular}
\caption{ES LAMBADA -tuned generator rollouts for the first Wikipedia article in the validation set. By the last step (119), the generator both learns the style of LAMBADA (stories), and the language (Spanish). It seems to learn the story style before switching languages.}

\end{table}

\begin{table}[t]
\centering
\underline{\textbf{\large LAMBADA ES, fasttext}}

\vspace{0.6em}
\rowcolors{2}{gray!10}{white}

\begin{tabular}{>{\centering\arraybackslash}p{0.14\linewidth} >{\raggedright\arraybackslash}p{0.76\linewidth}}
\toprule
\textbf{GRPO Step, $M$} & \textbf{Val Generation} \\
\midrule

4 &
\begin{minipage}[t]{\linewidth}
\begin{verbatim}
**Vardan Bostanjyan (also published as Vardan Bostanjyan (also
known as Bostanjyan) is an Armenian politician. He was born on
9 September 1949.**\n\n**Orange 2004**\n**Yerevan State
University Alumni Association Award**\n\n**Bostanjyan was
educated at Yerevan State University.** \n\n2006 elections\n**
Bostanjyan received a seat of the National Assembly from the
Prosperous Armenia party.**
\end{verbatim}
\end{minipage}
\\

29 &
\begin{minipage}[t]{\linewidth}
\begin{verbatim}
**Membro (Vardán Bartáyán Velía) Bostanjyan**\n\n**Nacimiento:
**\nNombre de nacimiento:** Burtel Artojan (desde 1949)\n**
Fecha de nacimiento:** 09 de setiembre de 1949\n**Héroe**: (Por
Honoridas fracciones cecesites)\n**Educación:** Universidad
Estatal de Yerevar (vida calle directoria)\n**Fuerzas de
trabajo:**\nº Argpectorense de los Spectres del Participado
Polítrico En el Congreso Nacional Allcisiano de Armenia "Ventura
proporcion
\end{verbatim}
\end{minipage}
\\

59 &
\begin{minipage}[t]{\linewidth}
\begin{verbatim}
 el.item y el año en el que se salvó el equipo en el año del
 suicido en el año del año las parido en el año el año en el año
 el año en el año del año en el año el año en el año el año el
 año en el año el año el año en el año el año el año el año el
 año en el año el año en el año el año el año en el año el año
 el año en el año en el año del año en el año el año el año en
 el año en el año en el año y por el año del año el año en el
 año en el
\end{verbatim}
\end{minipage}
\\

89 &
\begin{minipage}[t]{\linewidth}
\begin{verbatim}
Nombre del año del año en el año en el año en el año en el año en
el año para el año en el año para el año en el año para el año en
el año para el año del año para el año para el año en el año el año
en el año para el año para el año en el año para el año en el año
para el año para el año para el año para el gobierno en el año en
el año en el año en el año en el año en el año para el año para el
año en el año para el año para el año y hacia el año para el año
bajo el nombre del
\end{verbatim}
\end{minipage}
\\

119 &
\begin{minipage}[t]{\linewidth}
\begin{verbatim}
El nuevo año en el año y el año del año en el año en el año en el
año en el año y el año en el año en el año en el año en el año en
el año en el año en el año en el año en el año en el año en el año
en el año en el año en el año en el año en el año en el año en el
año en el año en el año en el año en el año ; el año en el año para
el año en el año en el año en el año en el año en el año en el año
para el año en el año para el año
\end{verbatim}
\end{minipage}
\\

\bottomrule
\end{tabular}
\caption{ES LAMBADA -tuned generator rollouts for the first Wikipedia article in the validation set, for our strongest non-metagradient baseline (fasttext language classification as the reward). There is a brief period where it looks like the generator has learned to translate its rephrase into the target language, but it quickly becomes degenerate.}

\end{table}

\begin{table}[t]
\centering
\underline{\textbf{\large UUID, Metagradients w Adam}}

\vspace{0.6em}
\rowcolors{2}{gray!10}{white}

\begin{tabular}{>{\centering\arraybackslash}p{0.14\linewidth} >{\raggedright\arraybackslash}p{0.76\linewidth}}
\toprule
\textbf{GRPO Step} & \textbf{Val Generation} \\
\midrule

4 &
\begin{minipage}[t]{\linewidth}
\begin{verbatim}
I can't fulfill this request.
\end{verbatim}
\end{minipage}
\\

59 &
\begin{minipage}[t]{\linewidth}
\begin{verbatim}
Here's a breakdown of the article:\n\n- **Author's Name**:
Vardan Bostanjyan\n- **Birthdate**: 9 September 1949\n-
**Birthplace**: Yerevan\n- **Occupation**: Politician\n-
**Years of service**: 1949-2003\n- **Positions held**: \n 
- Member of the National Assembly of Armenia (1999-2003)\n 
- Ambassador of the National Assembly of Armenia (2007)\n-
**References**: \n  - Yerevan State University alumni\n  -
Members of the National Assembly (Armenia)\n\nNow, let
\end{verbatim}
\end{minipage}
\\

109 &
\begin{minipage}[t]{\linewidth}
\begin{verbatim}
-000000000-0000-0000-0000-000000000000-0000-0000-0000-
000000000000-0000-0000-0000-000000000000-0000-0000-0000-
000000000000-0000-0000-0000-000000000000-0000-0000-0000-
000000000000-0000-0000-0000-000000000000-0000-0000-0000-
000000000000-0000-0000-0000-000000
\end{verbatim}
\end{minipage}
\\

159 &
\begin{minipage}[t]{\linewidth}
\begin{verbatim}
824747d3-69bf-46db-b4d0-8364b0b4b4fb \n648b4fb3-f4fb-4fb0-a4d0
-8364b0b4b4fb \n648b4fb3-f4fb-4fb0-a4d0-8364b0b4b4fb \n648b4fb
3-f4fb-4fb0-a4d0-8364b0b4b4fb \n648b4fb3-f4fb-4fb0-a4d0-
8364b0b
\end{verbatim}
\end{minipage}
\\

209 &
\begin{minipage}[t]{\linewidth}
\begin{verbatim}
664747d3-69bf-46db-b4d0-83649b5423fb\n\n## Step 1: Determine
the type of problem\nThis problem involves finding the maximum
value of a function that represents the number of ways to choose
a subset of items from a set of items with repetition allowed.
\n\n## Step 2: Define the function\nLet's define a function f(x)
that represents the number of ways to choose a subset of items
from a set of items with repetition allowed. The function can be
defined as follows:\nf(x) = (x + 1)^n, where n is the number of
\end{verbatim}
\end{minipage}
\\

259 &
\begin{minipage}[t]{\linewidth}
\begin{verbatim}
664747d3-69bf-46db-b4d0-83649b5423fb\n\n## Step 1:
Understand the problem\nThe problem is asking us to find
the value of x in the equation 2x + 5 = 11.\n\n## Step 2:
Isolate the variable\nTo isolate the variable x, we need
to get rid of the constant term on the right-hand side
of the equation. We can do this by subtracting 5 from
both sides of the equation.\n\n## Step 3: Simplify the
equation\nSubtracting 5 from both sides gives us 2x
\end{verbatim}
\end{minipage}
\\

309 &
\begin{minipage}[t]{\linewidth}
\begin{verbatim}
664747d3-69bf-46db-b4d0-83649b5423fb\n\n## Step 1:
Understand the problem statement\nThe problem statement
asks us to find the value of x in the quadratic equation
ax^2 + bx + c = 0.\n\n## Step 2: Recall the quadratic
formula\nThe quadratic formula is x = (-b ± U+221A(b^2 - 4ac))
/ 2a.\n\n## Step 3: Identify the coefficients of the
quadratic equation\nIn the quadratic equation ax^2 + bx +
c = 0, a = a, b = b,
\end{verbatim}
\end{minipage}
\\

359 &
\begin{minipage}[t]{\linewidth}
\begin{verbatim}
664747d-a3cf-46db-b4d0-83649b5423fb\n\n## Step 1:
Understand the problem statement\nThe problem asks
to find the value of x in the quadratic equation ax^2 +
bx + c = 0.\n\n## Step 2: Recall the quadratic formula\n
The quadratic formula is given by x = (-b ± U+221A(b^2 - 4ac))
/ 2a.\n\n## Step 3: Identify the coefficients of the
quadratic equation\nIn the quadratic equation ax^2 + bx
+ c = 0, a = a, b = b, and
\end{verbatim}
\end{minipage}
\\

\bottomrule
\end{tabular}
\caption{UUID-tuned generator rollouts for the first Wikipedia article in the validation set. By the last step (359), the generator has learned the UUID.}

\end{table}

\FloatBarrier

\section{Hypermarameters}
\label{app:hypers}

The following tables show the hyperparameter settings for all of our experiments. The frameworks we used were verl \citep{sheng2024hybridflow}, Hugging Face Transformers \citep{wolf-etal-2020-transformers}, JAX \citep{jax2018github}, PyTorch \citep{Ansel_PyTorch_2_Faster_2024}, and EasyDeL \citep{ZareChavoshi_2023}. Hyperparameters not specified for the generator can be assumed to be verl GRPO defaults as of August 28, 2025. We implemented a custom reward function in verl that ran $\mathcal{A}$ and computed metagrads for each batch of generator responses.

\begin{table}[t]
\centering

\underline{\textbf{\large QR Code, \ourmethodshort\ GRPO Hyperparameters}}

\vspace{0.6em}

\begin{tabular}{l c}
\toprule
\textbf{Parameter} & \textbf{Value} \\
\midrule

\multicolumn{2}{l}{\textbf{Generator}} \\
\midrule
\addlinespace[2pt]

Learning rate & $5\mathrm{e}{-6}$ \\
Max Prompt Length & 1024 \\
Max Response Length & 128 \\
Groups, $G$ & 4 \\
Rollout Batch Size / $G$ & 24576 \\
KL Coefficient & 0 \\
Train Temperature & 1.0 \\
Val Temperature & 1.0\\
GRPO Optimization Steps, $M$ & 200 \\
GRPO Train Epochs & 200\\
Model & \texttt{meta-llama/Llama-3.2-1B-Instruct} \\
Infra & verl, Hugging Face, PyTorch \\

\addlinespace[6pt]
\midrule
\multicolumn{2}{l}{\textbf{$\mathcal{A}$}} \\
\midrule
\addlinespace[2pt]

Learning rate & $5\mathrm{e}{-6}$ (Adam)\\
Adam $\beta_1$ & $0.9$ \\
Adam $\beta_2$ & $0.95$ \\
Adam $\epsilon$ & $1\mathrm{e}{-8}$ \\
Adam $\epsilon_{root}$ & $1\mathrm{e}{-9}$ \\
Weight Decay & $1\mathrm{e}{-4}$ \\
Train Steps, $T$ & 96 \\
Model & \texttt{gpt2} \\
Infra & EasyDeL, JAX \\

\bottomrule
\end{tabular}

\caption{Hyperparameters for the experiment in Figure \ref{fig:fig1}.}

\end{table}

\begin{table}[t]
\centering

\underline{\textbf{\large 67, \ourmethodshort\ GRPO Hyperparameters}}

\vspace{0.6em}

\begin{tabular}{l c}
\toprule
\textbf{Parameter} & \textbf{Value} \\
\midrule

\multicolumn{2}{l}{\textbf{Generator}} \\
\midrule
\addlinespace[2pt]

Learning rate & $5\mathrm{e}{-6}$ \\
Max Prompt Length & 1024 \\
Max Response Length & 128 \\
Groups, $G$ & 4 \\
Rollout Batch Size / $G$ & 256 (1), 2048 (8), 24576 (96) \\
KL Coefficient & 0 \\
Train Temperature & 1.0 \\
Val Temperature & 1.0\\
GRPO Optimization Steps, $M$ & 3840 (1), 480 (8), 40 (96) \\
GRPO Train Epochs & 40\\
Model & \texttt{meta-llama/Llama-3.2-1B-Instruct} \\
Infra & verl, Hugging Face, PyTorch \\

\addlinespace[6pt]
\midrule
\multicolumn{2}{l}{\textbf{$\mathcal{A}$}} \\
\midrule
\addlinespace[2pt]

Learning rate & $5\mathrm{e}{-6}$ (Adam),  $5.12\mathrm{e}{-4}$ (SGD), $2.56\mathrm{e}{-4}$ (Naive)\\
Adam $\beta_1$ & $0.9$ \\
Adam $\beta_2$ & $0.95$ \\
Adam $\epsilon$ & $1\mathrm{e}{-8}$ \\
Adam $\epsilon_{root}$ & $1\mathrm{e}{-9}$ \\
Weight Decay & $1\mathrm{e}{-4}$ \\
Train Steps, $T$ (Train Rollouts) & 1 (1), 8 (8), 96 (96) \\
Train Steps (Val Rollouts) & 96 \\
Model & \texttt{gpt2} \\
Infra & EasyDeL, JAX \\

\bottomrule
\end{tabular}

\caption{Hyperparameters for the 67 experiments. (1), (8), and (96) designate the (1), (8), and (96) variants of algorithm $\mathcal{A}$ that we test.}

\end{table}

\begin{table}[t]
\centering

\underline{\textbf{\large $\ell^2$ Norm, \ourmethodshort\ GRPO Hyperparameters}}

\vspace{0.6em}

\begin{tabular}{l c}
\toprule
\textbf{Parameter} & \textbf{Value} \\
\midrule

\multicolumn{2}{l}{\textbf{Generator}} \\
\midrule
\addlinespace[2pt]

Learning rate & $5\mathrm{e}{-6}$ \\
Max Prompt Length & 1024 \\
Max Response Length & 128 \\
Groups, $G$ & 4 \\
Rollout Batch Size / $G$ & 256 (1), 2048 (8), 24576 (96)\\
KL Coefficient & 0 \\
Train Temperature & 1.0 \\
Val Temperature & 1.0\\
GRPO Optimization Steps, $M$ & 960 (1), 120 (8), 10 (96)\\
GRPO Train Epochs & 10\\
Model & \texttt{meta-llama/Llama-3.2-1B-Instruct} \\
Infra & verl, Hugging Face, PyTorch \\

\addlinespace[6pt]
\midrule
\multicolumn{2}{l}{\textbf{$\mathcal{A}$}} \\
\midrule
\addlinespace[2pt]

Learning rate & $5\mathrm{e}{-6}$ (Adam),  $1.28\mathrm{e}{-4}$ (SGD), $1\mathrm{e}{-6}$ (Naive)\\
Adam $\beta_1$ & $0.9$ \\
Adam $\beta_2$ & $0.95$ \\
Adam $\epsilon$ & $1\mathrm{e}{-8}$ \\
Adam $\epsilon_{root}$ & $1\mathrm{e}{-9}$ \\
Weight Decay & $1\mathrm{e}{-4}$ \\
Train Steps, $T$ (Train Rollouts) & 1 (1), 8 (8), 96 (96) \\
Train Steps (Val Rollouts) & 96 \\
Model & \texttt{gpt2} \\
Infra & EasyDeL, JAX \\

\bottomrule
\end{tabular}

\caption{Hyperparameters for the $\ell^2$ Norm experiments. (1), (8), and (96) designate the (1), (8), and (96) variants of algorithm $\mathcal{A}$ that we test.}

\end{table}

\begin{table}[t]
\centering

\underline{\textbf{\large LAMBADA, \ourmethodshort\ GRPO Hyperparameters}}

\vspace{0.6em}

\begin{tabular}{l c}
\toprule
\textbf{Parameter} & \textbf{Value} \\
\midrule

\multicolumn{2}{l}{\textbf{Generator}} \\
\midrule
\addlinespace[2pt]

Learning rate & $1\mathrm{e}{-6}$ \\
Max Prompt Length & 1024 \\
Max Response Length & 128 \\
Groups, $G$ & 4 \\
Rollout Batch Size / $G$ & 256 \\
KL Coefficient & 0 \\
Train Temperature & 1.0 \\
Val Temperature & 1.0\\
GRPO Optimization Steps, $M$ & 120\\
GRPO Train Epochs & 3\\
Model & \texttt{meta-llama/Llama-3.2-1B-Instruct} \\
Infra & verl, Hugging Face, PyTorch \\

\addlinespace[6pt]
\midrule
\multicolumn{2}{l}{\textbf{$\mathcal{A}$}} \\
\midrule
\addlinespace[2pt]

Learning rate & $1\mathrm{e}{-6}$ (Adam),  $6.4\mathrm{e}{-5}$ (SGD), $6.4\mathrm{e}{-5}$ (Naive)\\
Adam $\beta_1$ & $0.9$ \\
Adam $\beta_2$ & $0.95$ \\
Adam $\epsilon$ & $1\mathrm{e}{-8}$ \\
Adam $\epsilon_{root}$ & $1\mathrm{e}{-9}$ \\
Weight Decay & $1\mathrm{e}{-4}$ \\
Train Steps, $T$ & 1 \\
Model & \texttt{meta-llama/Llama-3.2-1B-Instruct} \\
Infra & EasyDeL, JAX \\

\bottomrule
\end{tabular}

\caption{Hyperparameters for the LAMBADA maintext experiments.}

\end{table}

\begin{table}[t]
\centering

\underline{\textbf{\large Hyperparameters for CPT on LAMBADA-tuned Generator Outputs}}

\vspace{0.6em}

\begin{tabular}{l c}
\toprule
\textbf{Parameter} & \textbf{Value} \\
\midrule
\addlinespace[2pt]

Learning rate & $1\mathrm{e}{-6}$\\
Adam $\beta_1$ & $0.9$ \\
Adam $\beta_2$ & $0.95$ \\
Adam $\epsilon$ & $1\mathrm{e}{-8}$ \\
Weight Decay & $0$ \\
Context Length & $128$ \\
Train Tokens & 10M \\
Warmup Ratio & $0.1$\\
LR Schedule & Cosine\\
Batch Size & $512$ \\
Model & \texttt{meta-llama/Llama-3.2-1B-Instruct}, \texttt{meta-llama/Llama-3.2-1B} \\
Replay Data & DCLM Baseline \\
Replay Ratio & $0.1$ \\
Infra & Hugging Face\\

\bottomrule
\end{tabular}

\caption{Hypermarameters for the CPT experiments shown in Table \ref{tab:multilingualcpt}.}

\end{table}

\begin{table}[t]
\centering

\underline{\textbf{\large Hyperparameters for SFT on LAMBADA}}

\vspace{0.6em}

\begin{tabular}{l c}
\toprule
\textbf{Parameter} & \textbf{Value} \\
\midrule
\addlinespace[2pt]

Learning rate & $1\mathrm{e}{-6}$, $1\mathrm{e}{-5}$\\
Weight Decay & $0$ \\
Epochs & $1$, $5$, $10$ \\
Context Length & $2048$ \\
Batch Size & $64$ \\
Model & \texttt{meta-llama/Llama-3.2-1B-Instruct}, \texttt{meta-llama/Llama-3.2-1B} \\
Infra & Hugging Face\\

\bottomrule
\end{tabular}

\caption{Hypermarameters for the SFT experiments shown in Table \ref{tab:multilingualcpt}. We conducted a grid search over the multiple values shown for learning rate and epochs, and chose the best performing SFT model on the validation set for each benchmark.}
\label{tab:sfthparams}

\end{table}

\begin{table}[t]
\centering

\underline{\textbf{\large UUID, \ourmethodshort\ GRPO Hyperparameters}}

\vspace{0.6em}

\begin{tabular}{l c}
\toprule
\textbf{Parameter} & \textbf{Value} \\
\midrule

\multicolumn{2}{l}{\textbf{Generator}} \\
\midrule
\addlinespace[2pt]

Learning rate & $1\mathrm{e}{-6}$ \\
Max Prompt Length & 1024 \\
Max Response Length & 128 \\
Groups, $G$ & 4 \\
Rollout Batch Size / $G$ & 256 \\
KL Coefficient & 0 \\
Train Temperature & 1.0 \\
Val Temperature & 0.0\\
GRPO Optimization Steps, $M$ & 360\\
GRPO Train Epochs & 9\\
Model & \texttt{meta-llama/Llama-3.2-1B-Instruct} \\
Infra & verl, Hugging Face, PyTorch \\

\addlinespace[6pt]
\midrule
\multicolumn{2}{l}{\textbf{$\mathcal{A}$}} \\
\midrule
\addlinespace[2pt]

Learning rate & $1\mathrm{e}{-6}$ (Adam),  $6.4\mathrm{e}{-5}$ (SGD), $6.4\mathrm{e}{-5}$ (Naive)\\
Adam $\beta_1$ & $0.9$ \\
Adam $\beta_2$ & $0.95$ \\
Adam $\epsilon$ & $1\mathrm{e}{-8}$ \\
Adam $\epsilon_{root}$ & $1\mathrm{e}{-9}$ \\
Weight Decay & $1\mathrm{e}{-4}$ \\
Train Steps, $T$ & 1 \\
Model & \texttt{meta-llama/Llama-3.2-1B-Instruct} \\
Infra & EasyDeL, JAX \\

\bottomrule
\end{tabular}

\caption{Hyperparameters for the UUID experiments.}

\end{table}

\FloatBarrier

\section{GPT 4.1 Nano Prompt}
\label{app:gptprompt}
We use the following prompt to GPT 4.1 Nano \citep{openai_gpt41_nano} to decide whether the language is correct in Figure \ref{fig:langresults}. We replace \texttt{<excerpt>} with generator responses.

\begin{tcolorbox}[colback=gray!15, colframe=gray!15, boxrule=0pt]
\begin{verbatim}
Classify the language of the following excerpt. 
Your answer must be the best choice of:
English, Spanish, German, Italian, French, Not Natural Language. 
Output only your final choice with no explanation. 
Here is the excerpt: <excerpt>
\end{verbatim}
\end{tcolorbox}

\section{Wikipedia Paraphrase Prompt}
\label{app:rephraseprompt}
We use the following prompt for our generator, where \texttt{<article>} is replaced with Wikipedia articles to paraphrase. Due to the prompt length limit (see Appendix \ref{app:hypers}), the article is often truncated.

\begin{tcolorbox}[colback=gray!15, colframe=gray!15, boxrule=0pt]
\begin{verbatim}
Help read the following article and then rephrase it in
different terms. Remember to keep the meaning and every
content of the article intact, including the title, year,
etc. Here is the article:\n<article>
\end{verbatim}
\end{tcolorbox}

\FloatBarrier

\end{document}